\documentclass[10 pt, journal, twoside]{IEEEtran}

\usepackage{setspace}

\usepackage{graphics} 
\usepackage{adjustbox}
\usepackage{amsmath,bm} 
\usepackage{amsthm} 
\usepackage{amssymb}  
\usepackage{booktabs} 
\newcommand{\ra}[1]{\renewcommand{\arraystretch}{#1}} 
\usepackage{cases}
\usepackage{siunitx}
\usepackage[shortcuts]{extdash}
\usepackage{multicol}
\usepackage{multirow}
\usepackage{tabstackengine}
\stackMath
\usepackage[table,dvipsnames]{xcolor}
\usepackage{trfsigns}
\usepackage{caption}
\usepackage{subcaption}
\usepackage{paralist} 
\usepackage{enumitem}
\usepackage[nolist]{acronym}
\usepackage{soul}
\usepackage{tikz}
\usepackage{verbatim}
\usepackage{placeins}

\usepackage{hyperref}  

\colorlet{mylinkcolor}{violet}
\colorlet{mycitecolor}{YellowOrange}
\colorlet{myurlcolor}{Aquamarine}
\colorlet{myblack}{Black}

\hypersetup{
  linkcolor  = myblack,
  citecolor  = myblack,
  urlcolor   = myblack,
  colorlinks = true,
}

\usepackage{custom/chb_colors}
\usepackage{custom/chb_commands}
\usepackage{custom/chb_robo_var_defs}

\definecolor{t10}{HTML}{67000d}
\definecolor{t9}{HTML}{67000d}
\definecolor{t8}{HTML}{a50f15}
\definecolor{t7}{HTML}{cb181d}
\definecolor{t6}{HTML}{ef3b2c}
\definecolor{t5}{HTML}{fb6a4a}
\definecolor{t4}{HTML}{fc9272}
\definecolor{t3}{HTML}{fcbba1}
\definecolor{t2}{HTML}{fee0d2}
\definecolor{t1}{HTML}{fff5f0}
\newcommand{\mpsympair}{
\mbox{
\begin{tikzpicture}[baseline=-0.6ex, overlay]
\draw[black,fill=cyan] (0,0) circle (.8ex);node[midway];
\end{tikzpicture}
\hspace{0.0em} 
}
}

\newcommand{\mpsymopen}{
 \mbox{
 \begin{tikzpicture}[baseline=-0.6ex,overlay]
 \filldraw[black,fill=violet] (-0.7ex,-0.6ex) -- (0ex,0.6ex) -- (.7ex,-0.6ex) -- cycle;node[midway];
 \end{tikzpicture}
 \hspace{0.0em} 
 }
}

\newcommand{\mpsymrecall}{
\mbox{
\begin{tikzpicture}[baseline=-0.6ex, overlay]
\draw[green,fill=green] (0em,0) circle (.8ex);node[midway];
\end{tikzpicture}
\hspace{0em} 
}
}

\graphicspath{{./}{figures/}}

\title{Sensor Model Identification via Simultaneous Model Selection and State Variable Determination}

\author{Christian Brommer$^{1}$, Alessandro Fornasier$^{1}$, Jan Steinbrener$^{1}$, Stephan Weiss$^{1}$%

\thanks{$^{1}$ Control of Networked Systems Group, University of Klagenfurt, Austria. E-Mail:
        {\tt\footnotesize \href{mailto:christian.brommer@aau.at}{christian.brommer@aau.at}},\\
        {\tt\footnotesize { \{ }}
        {\tt\footnotesize \href{mailto:alessandro.fornasier@ieee.org}{alessandro.fornasier}},
        {\tt\footnotesize \href{mailto:jan.steinbrener@ieee.org}{jan.steinbrener}},\\
        {\tt\footnotesize \href{mailto:stephan.weiss@ieee.org}{stephan.weiss}}
        {\tt\footnotesize { \}@ieee.org} }
        }%
\thanks{Corresponding author: Christian Brommer}%
\thanks{Research was sponsored by the Army Research Office and was accomplished under Cooperative Agreement Number W911NF-21-2-0245. The views and conclusions contained in this document are those of the authors and should not be interpreted as representing the official policies, either expressed or implied, of the Army Research Office or the U.S. Government. The U.S. Government is authorized to reproduce and distribute reprints for Government purposes notwithstanding any copyright notation herein.}
\thanks{This work has been accepted for publication in IEEE Transactions on Robotics (T-RO). $\copyright$ 2025 IEEE.  Personal use of this material is permitted.  Permission from IEEE must be obtained for all other uses, in any current or future media, including reprinting/republishing this material for advertising or promotional purposes, creating new collective works, for resale or redistribution to servers or lists, or reuse of any copyrighted component of this work in other works.}
\thanks{\textbf{The DOI will be added when available.}}%
}

\newtheorem{theorem}{Theorem}[section]

\begin{document}

\maketitle
\begin{acronym}
\acro{MAV}[MAV]{Micro Aerial Vehicle}
\acro{UAV}[UAV]{Unmanned Aerial Vehicle}
\acro{GNSS}[GNSS]{Global Navigation Satellite System}
\acro{RTK}[RTK]{Real-Time Kinematic}
\acro{IMU}[IMU]{Inertial Measurement Unit}
\acro{LRF}[LRF]{Laser Range Finder}
\acro{UWB}[UWB]{Ultra-Wide-Band}
\acro{FoV}[FoV]{Field of View}
\acro{VIO}[VIO]{Visual Inertial Odometry}
\acro{EMI}[EMI]{Electromagnetic Interference}
\acro{MEMS}[MEMS]{Micro-Electromechanical Systems}
\acro{SoC}[SoC]{System on Chip}
\acro{FoV}[FoV]{Field of View}
\acro{UHF}[UHF]{Ultra High Frequency}
\acro{SfM}[SfM]{Structure-from-Motion}
\acro{GPIO}[GPIO]{General Purpose Input/Output}
\acro{EKF}[EKF]{Extended Kalman Filter}
\acro{CSE}[CSE]{Collaborative State Estimation}
\acro{RFID}[RFID]{Radio Frequency Identification}
\acro{LASSO}[LASSO]{Least Absolute Shrinkage and Selection Operator}
\acro{LARS}[LARS]{Least Angle Regression}
\acro{SINDy}[SINDy]{Sparse Identification of Non-linear Dynamics}
\acro{AIC}[AIC]{Akaike Information Criterion}
\acro{BIC}[BIC]{Bayesian Information Criterion}
\acro{SR3}[SR3]{Sparse Relaxed Regularized Regression}
\acro{PSO}[PSO]{Particle Swarm Optimization}
\end{acronym}

%
\begin{abstract}
We present a method for the unattended gray-box identification of sensor models commonly used by localization algorithms in the field of robotics. 
The objective is to determine the most likely sensor model for a time series of unknown measurement data, given an extendable catalog of predefined sensor models. Sensor model definitions may require states for rigid-body calibrations and dedicated reference frames to replicate a measurement based on the robot's localization state. 
A health metric is introduced, which verifies the outcome of the selection process in order to detect false positives and facilitate reliable decision-making. 
In a second stage, an initial guess for identified calibration states is generated, and the necessity of sensor world reference frames is evaluated.
The identified sensor model with its parameter information is then used to parameterize and initialize a state estimation application, thus ensuring a more accurate and robust integration of new sensor elements.
This method is helpful for inexperienced users who want to identify the source and type of a measurement, sensor calibrations, or sensor reference frames. It will also be important in the field of modular multi-agent scenarios and modularized robotic platforms that are augmented by sensor modalities during runtime. Overall, this work aims to provide a simplified integration of sensor modalities to downstream applications and circumvent common pitfalls in the usage and development of localization approaches.
\end{abstract}%
\vspace{-0.2cm}


\section{Introduction}
\label{sec:introduction}
Accurate and robust localization is a crucial prerequisite for any application of autonomous robotics. Without accurate localization information, the generation of control inputs will not be possible. Localization algorithms in research and consumer products heavily rely on a number of sensor types to infer the most accurate position and orientation possible.
In addition to position and orientation, state-of-the-art robot localization approaches such as \acp{EKF} can also estimate additional states such as sensor calibrations (e.g., bias offsets or extrinsic calibration parameters) online.
For this reason, essential core states (position, velocity, orientation, etc.) and sensor calibration states $\mathbf{x}$ are incorporated in the sensor model $h(\mathbf{x})$ and the Jacobian for sensor measurement updates $\mathbf{H}(\mathbf{x}) = \frac{\partial h(\mathbf{x})}{\partial \mathbf{x}}$.

Figure~\ref{fig:sensor_frames} shows the essential configuration of the sensor models that are addressed by the presented approach. A full definition of all sensor models is shown in Section~\ref{sec:method}.
The integration of such sensors can be a complex task, and typically, a number of questions need to be answered prior to the integration:
\begin{itemize}[leftmargin=*,noitemsep,topsep=0pt,parsep=0pt,partopsep=0pt]
    \item What type of information does the sensor provide?
    \item Given a rigid body, what is the extrinsic calibration?
    \item Is the measurement given in the normal direction ($\mvar{T}[sr]$) or the inverse with respect to the robot ($\mvar{T}[sr]^{-1}$) (see Fig.~\ref{fig:sensor_frames})?
    \item Does the measurement model require a reference frame?
\end{itemize}

The integration of a sensor modality usually requires expert knowledge. In order to generate a sensor model for the update in filter definitions, the main components of the general system, such as given and prior states, calibration states, and reference frames, need to be identified.

\begin{figure}[!t]
    \centering
    \includegraphics[width=1.0\linewidth, trim={0cm 0mm 0cm 0mm},clip]{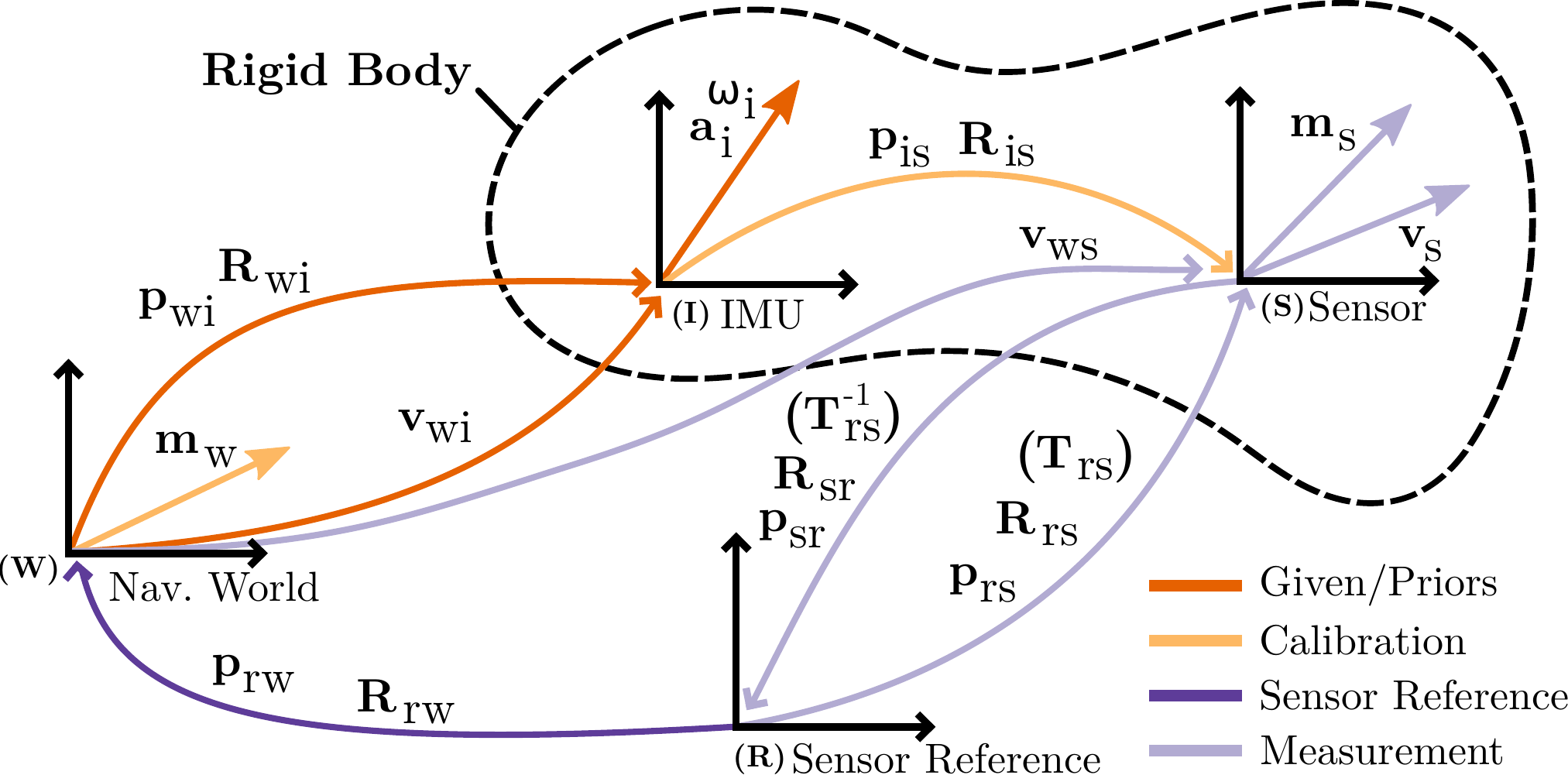}
    \caption{Generalized setup of transformations to express various sensor modalities such as vectors for the magnetic field or the velocity and transformations for 3 and 6-DoF sensor measurements and calibrations.
    This work aims to process a gray-box sensor signal together with a reliable system state to identify a corresponding sensor model and its properties.}
    \label{fig:sensor_frames}
    \vspace{-5mm}
\end{figure}

To provide additional and quantifiable insight into the magnitude of complications related to the setup of localization and state-estimation solutions, we evaluated the number of relevant issues, discussions, and wikis for the open-source platform GitHub\footnote{Data collected from GitHub repositories and discussion boards, open and closed issues, as of December 2024.} with respect to sensor-modeling and state-estimation. The keyword ``Sensor model'' appears in 80k issues, 4k discussions, and 13k wikis, with a significant share of 14.4k issues and 1k discussions related to ``sensor model integration''. Similarly, ``State-Estimation'' (30k issues), ``sensor-fusion'' (8.5k issues), and ``sensor calibration'' (22k issues). This indicates that setting up localization solutions and integrating sensor modalities, 1) requires extensive documentation - to be read, and 2) is not always straightforward and causes predominant and time-consuming issues. Such issues even motivate setups where raw sensor data or entirely proprietary solutions are preferred for ease of use.
The proposed work aims to allow a simplified and automated integration of sensor modules during runtime without prior knowledge about the sensor type or its properties to streamline and robustify the task of integrating a new sensor model. This mitigates possible errors and accelerates the research and development phases.

In previous work \cite{Brommer2021_mars}, we focused on the modularity aspect of a possible application that integrates a sensor modality during runtime. The important aspect is that a sensor module does not need to be \textit{a priori} known to the system.
However, this approach requires knowledge of the sensor type and, ideally, an initial guess for the additional calibration states of a sensor module.

Of course, a pragmatic approach is to always know the sensor modalities that are used during the operation of a robot, or having a self-identifying sensor.
However, we aim to simplify the process of sensor integration without expert knowledge and future work that does not rely on knowing a sensor type before integration. We illustrate two scenarios where the proposed method could be applied:

\paragraph{Automated Robust Sensor Integration}
Sensor devices are not always well documented nor easily integrated by an engineer without prior knowledge about state-estimation algorithms and specific sensor modalities. Assuming an inexperienced person needs to integrate the odometry information provided by an Intel Real-Sense sensor, which is rigidly mounted on a UAV.

First, we need to determine the type and direction of the measurements. In this case, the position and rotation measurements could be expressed w.r.t. an arbitrary vision frame chosen by the sensor up on initialization, its inverse, or different directions for the position and rotation information. Given a lack of documentation, arbitrary reference frames, or inexperience, finding the right references and directions can result in time-consuming tabletop experiments.

The next step is to determine the calibration states and reference frame of the sensor. We assume that the UAV uses an IMU as a core sensor, and we need to find the rigid-body calibration of the Real-Sense w.r.t. the IMU frame. Lack of knowledge on how the IMU is mounted, how the odometry frame of the Real-Sense is oriented, and which arbitrary reference frame is used by the Real-Sense can lead to errors when inferring this calibration through estimation by hand or evaluation of numerical measurements of the IMU and the Real-Sense.
\looseness=-1

We also need to know the reference in which the measurement is expressed and if this differs from the currently used world reference of the localization algorithm. If the reference frame does not differ, then it does not need to be incorporated into the sensor model.
In the case of the Real-Sense, the odometry reference is most likely aligned with the gravity vector, but the remaining orientation parameters and the position depend on the pose of the sensor during internal initialization.

These elements can lead to time-consuming integration and error-prone results. With the proposed method, the sensor type and parameters for its setup are inferred automatically without the need of expert knowledge, thus simplifying the system integration.

\paragraph{\ac{CSE}} To the contrary of the previous example, for \ac{CSE}, localization information of one robot is not only given by rigidly attached sensor modalities. Localization information can also be inferred through temporary encounters with other moving robotic platforms, sharing their state estimates in the vicinity.

One possible use-case is a setup with an \ac{UAV} and rover, providing a recharging base for the \ac{UAV}. While charging, the UAV is temporarily attached to the moving rover and can either share its state estimate or individual sensor modalities with the rover.
In this case, sensor modalities, respective calibration states, and references do not need to be known a priory. Gray-box sensor information can be transmitted to the rover, which identifies the sensor module, type of measurement, and sensor calibrations.
During this encounter, both platforms can benefit and improve their estimated localization.

This example can be transferred to a more tangible industry use case, given a scenario with an autonomous pallet truck that picks up boxes with individual \acp{RFID} modules that can be located within the warehouse. Given the presented method and a short measurement sequence, the localization algorithm of the pallet truck can incorporate the measurement automatically without further knowledge, and benefit from the additional position measurements, improving accuracy and redundancy.

Of course, a more practical approach is to pre-process and unify the output of collaborating robots and to use this information directly. However, if a robot carries more than one sensor, then sharing individual sensor information and incorporating them with accurately modeled extrinsic calibration states in a tightly coupled fashion can improve robustness, accuracy, and consistency, given that probabilistic correlations are estimated accurately. Depending on the overall complexity of such a modular system, a self-identification approach to sensor models can also be more scalable.
In any case, it is helpful if the correct sensor model, and its parameters (i.e., calibrations and references) are inferred automatically, based on the given measurement signal, following the presented method outlined by Figure~\ref{fig:method_outline} to efficiently integrate sensor information online into a localization algorithm.

\noindent The presented work focuses on the following contributions:
\begin{itemize}[leftmargin=*,noitemsep,topsep=0pt,parsep=0pt,partopsep=0pt]
    \item Gray-Box sensor model identification based on noisy estimates of localization states and sensor measurement signals from an unknown sensor type. The measurement source is unknown, and a catalog of the most common sensor models in the field of robotics is provided.
    \item Estimation of sensor calibration states for filter initialization.
    \item Evaluation of the requirement for a dedicated sensor measurement reference frame, needed to align the sensor information with the navigation frame of a localization algorithm.
    \item A health metric for the validation of a selected sensor module, preventing false positives.
\end{itemize}

\noindent Essential steps of the proposed sensor identification method are illustrated in Figure~\ref{fig:method_outline} and further detailed in Section~\ref{sec:method}.
The general structure of the method is as follows. The localization system needs to be in a self-sustained state, meaning a global or stable relative localization needs to be available. In the presented examples, this information is provided by an EKF with relatively noise core states (i.e., an IMU-driven propagation with vision sensor updates).
Stage~1 is responsible for the detection of the correct sensor model and the correct sensor signal references, e.g., is the sensor measurement expressed w.r.t. the world frame, or is the measurement expressed by the inverse? This is done based on an over-determined system definition and provides a health metric to prevent false positive selections.
Stage~2 performs the estimation of the calibration states for the selected sensor model and identifies if a sensor reference frame, different from the world reference frame, is needed. After the sensor model and the existence of additional states are known, the correct Jacobian for the integration of an updated sensor to a filter framework is selected.

\begin{figure}[!t]
    \centering
    \includegraphics[width=1.0\linewidth, trim={0 0 0 0},clip]{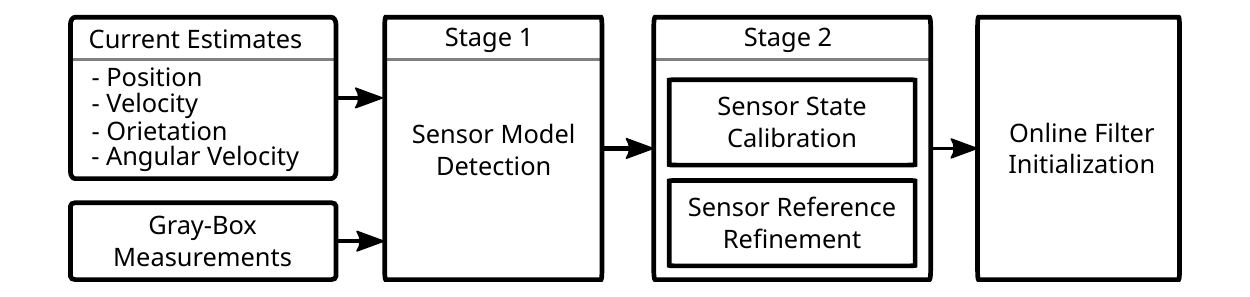}
    \caption{Outline of the stages for the sensor model identification method.}
    \label{fig:method_outline}
\vspace{-6mm}
\end{figure}

In the final step, the sensor model and the parameterization of the initial calibration states are integrated as an additional source of information to a state estimation framework. The sensor self-calibration of the previously presented framework ensures the refinement of the calibration states in the long term.
A complete integration of this process was tested with simulated data in Section~\ref{sec:sensitivity} and with real-world data and two robotic platforms: A UAV use case (see Sec.~\ref{sec:experiment_real_world_uav}) and a ground vehicle use case (see Sec.~\ref{sec:experiment_real_world_ground_vehicle}).
%

\section{Related Work}
\label{sec:related_work}
The presented work generally concerns the research on model identification based on given sensor data and partially known system models.
A number of methods have been developed and address solutions to this problem in various ways.
\looseness=-1

A pragmatic approach was summarized by Mitterer \cite{Mitterer2020}, where a sensor contains an electronic data sheet, accessible to a host, detailing the properties of a sensor. While this is a great solution for higher-level applications, fundamental non-smart sensors do not provide this functionality, and the determination of an extrinsic calibration still requires a dedicated process. 

Analytical methods have been introduced, starting with Fuzzy Modeling. Nagai \cite{Nagai2005} and Frolik \cite{Frolik2001} describe the generation of a fuzzy logic sensor model for a number of inputs and a corresponding output, additionally exposing information from possible latent features. Nagai \cite{Nagai2005} uses uncategorized input information to generate a refined model, tailored to a specific output, such as to predict variables in a chemical process where direct measurements are not possible.
The approach is especially interesting because it allows different criteria on the variable sets, such as Gaussian distributions, and the clustering of data points (Fuzzy c-means) to coherent structures by simplification with Kohonen maps. The clustering achieves a simplified set of sensor states and promotes resilience to over-fitting. 

However, as mentioned in the introduction, the goal is to generate a model that can be used in a filter setup with online self-calibration.
Since this generated model does not provide a direct association with the physical states of a sensor model, an abstract Fuzzy model will not allow human interpretation of physical calibration states and would provide a solution with untraceable states.

A more traceable solution for identifying states or coefficients and their significance was introduced by Tibshirani~\cite{Tibshirani1996} with the \acf{LASSO} and later adaptation \acf{LARS} by Efron~\cite{Efron2004}. LASSO poses a regularization in the form of the L1-Norm on a selection of coefficients to an optimization problem, thus promoting sparsity of coefficients. A selection is made by eliminating coefficients with insignificant contributions to the overall solution.
To utilize this approach, an over-determined system must be defined and evaluated against a given stream of e.g. sensor signals.
This approach was not designed to identify individual states of higher dimensions, and each coefficient is evaluated individually, without definable coherence between states.
This renders the usage of the LASSO approach difficult for usage in multi-dimensional geometric problems. It is not guaranteed that a non-required vector in $\mathbb{R}^{3}$ is eliminated entirely, especially if the signal has noise, allowing for possible over-fitting.

The LASSO operator is one of the main techniques used by the \acf{SINDy} framework by Brunton~\cite{Brunton2016}. SINDy was used for the identification of time-varying physical problems and demonstrated with fluid dynamical problems \cite{Loiseau2018}.
It requires a time series of data and over-determined model expressions that are commonly present for the problems that are to be analyzed.
This approach was extended by Mangan~\cite{Mangan2017} to incorporate information criteria (\ac{AIC} and \ac{BIC}) for model verification and to find the most informative model through sequential thresholding. 
Mangan~\cite{Mangan2017} also discussed the problem that noise-afflicted data could prevent the identification of the correct model with SINDy and the possible requirement of pre-filtered data.

Zheng~\cite{Zheng2019} also extended the work on SINDy by introducing the \ac{SR3} Framework. This extension introduced a relaxation loss term, which aims to increase the robustness of the state sparsification if inaccurate or noise-afflicted data is given.
The sparsification is not performed on the coefficients $\mathbf{x}$ directly but via a second loss term $\mathbf{w}$ which allows the relaxation 
\mbox{$\min_{x,w}~\frac{1}{2}\lVert \mathbf{Ax}-\mathbf{b} \rVert^2_2 + \lVert  \mathbf{w} \rVert_1+ \lVert \mathbf{x}-\mathbf{w} \rVert^2_2$} with system definition $\mathbf{A}$ and input $\mathbf{b}$.
While SINDy can solve multi-dimensional problems, it was not directly designed to incorporate multi-dimensional states such as geometrical transformations.

Other approaches for model identification perform iterative optimization steps.
Mixed integer non-linear programming is one technique; it also requires a predefined and meaningful catalog of model components and estimates coefficients and boolean-like or integer variables using a branch and bound technique.
Integer variables are selected iteratively (branch), and the value range is limited to a specific range (bound). Variables that do not minimize the problem in the value range are reversed, and the next branch of variable thresholding is followed.

This approach is similar to Bayesian networks, which also utilize a branching technique, with the difference that individual states are added without a decision variable as a weight, and the choice of the addition or existence of a state is made based on the improvement of the system error.
Kadane~\cite{Kadane2004} and Khosbayar~\cite{Khosbayar2022} describe multiple approaches to model selection by utilizing the Bayesian approach. Compared to the optimization-based methods, which require penalties that promote sparsification, e.g., regularization through LASSO and LARS or multiplicative mixed integers, the Bayesian approach includes or rejects model states for each iteration.
This results in a similar methodology of sparsifying a sensor model while maintaining accuracy.

A drawback with branching techniques is that the order in which states are added and evaluated can lead to non-deterministic and possibly wrong results depending on the nature of the proposed system models.
An additional state that improves the system's accuracy in first iterations does not necessarily belong to the correct sensor model.
Even with forward selection and backward elimination methods, the final model selection might minimize the error but remain incorrect. Reasons for false detections can be similarities of the sensor models (see~Eq.~\eqref{eq:mag_detail}~and~\eqref{eq:rot_detail}) or differences in magnitudes of individual state variables (see~Sec.~\ref{sec:method}). Thus, an incorrect first choice and the addition of a state can be irreversible.

A brute-force method for model identification through Boolean decision-making would be to iterate through a binary table for all posed states one by one. However, depending on the number of possible states, this approach can be too time-consuming/inefficient.
Lesage~\cite{Lesage-Landry2021} uses a probabilistic approach with randomly drawn Bernoulli samples to find a suitable selection more efficiently,
while Xiong~\cite{Xiong2021} updates Boolean definitions based on a thresholding technique and validates decisions based on the iterative system error.

In summary, Fuzzy logic finds abstract states that correlate given input data to output data. However, it does not provide an association to physically meaningful states, which prevents online state estimation in subsequent stages.
LASSO regularizes the collective magnitude of a state vector and reduces the number of states (selection) for the solution of a problem. However, it does not allow structured group penalties of multi-dimensional variables, making the selection process inaccurate for geometrical problems.
SINDy finds a polynomial for a dynamic system model, and variations to this approach address further improvements to noise resilience; however, incorporating states that perform multi-dimensional geometrical transformations is not straightforward.

Furthermore, approaches with iterative refinement stages divide the solution to model selection into two stages, the state selection and the state estimation in consecutive steps, where the state selection can change per iteration after state values have been refined. Such approaches can be time-consuming and lead to incorrect state selections if the posed model states are similar and prone to over-fitting in the presence of noise.

To the best of our knowledge, no existing approach or framework addresses the identification of multi-dimensional, geometrically constrained sensor models, nor the detection of false positives as presented by this work.
Our approach contributes to the existing literature, by allowing for a direct selection of interpretable sensor models that can incorporate multi-dimensional geometric transformations and a system definition for the optimization of selection operators and refinement of state variables simultaneously. It also establishes a health metric for the identification of false positive selections, and resilience to signal scaling and noise.
%

\section{Notation}
\label{sec:notation}
For a generalized description of sensor models, we are defining frame dependencies as shown by Figure \ref{fig:sensor_frames}.
Following the notation $\vvar{p}[B][C][A]$ for a translation of frame~\rframe{C} with respect to frame~\rframe{B} expressed in frame~\rframe{A}. 
With
$\vvar{p}[B][C][A] \equiv \vvar{p}[B][C]$ if frame \rframe{A}=\rframe{B}.
The quaternion~$\qvar{A}{B}$ describes the rotation of frame~\rframe{B} with respect to frame~\rframe{A} and corresponds to a rotation matrix ${\mathbf{R}_{\left(\qvar{A}{B}\right)} \equiv \rvar{A}{B}}$.
A variable describing a time sequence is denoted by the prefix $t$, e.g., $\tvar{v}$. Since the transformations in Figure~\ref{fig:sensor_frames} are used for multiple sensors, such states are labeled with the sensor type ${}^\text{type}\mathbf{x}$, e.g.${}^\mathsf{p}\mathbf{p}_{\mathsf{is}}$ for the translation of the sensor frame w.r.t. the IMU frame used for the position sensor (\rframe{P}) (see labels in Tab.~\ref{tab:state_sharing}).

Rotations utilized within the system definition are defined as Rotation matrices in $SO\left(3\right)$. Rotations as states which are optimized, or measurements, are defined in the tangent space with conversions $\log(\cdot)^\vee : SO\left(3\right) \rightarrow \mathbb{R}^{3}$  and $\exp(\boldsymbol{\omega}^\wedge):\mathbb{R}^{3} \rightarrow SO\left(3\right)$ with $(\cdot)^\wedge \equiv \lfloor\cdot\rfloor_\mathsf{x}$ and $(\cdot)^\vee$ as the structural decomposition of the skew-symmetric structure.
This definition ensures the correct properties of rotations during the optimization process.

\section{Method}
\label{sec:method}
This section describes how an over-determined system model is defined, which aspects are critical for selection operators, and which constraints are important to render the presented method reliable.
As illustrated by Figure~\ref{fig:method_outline}, the proposed method requires a given sequence of associated (synchronized) data points providing the current state estimate and the measurement by the unknown sensor module as the input.
The two data streams can be acquired, e.g., during the operation of the robot and after a sensor was added for online integration, or in a handheld setting for offline processing.
After a data sequence was acquired, the data points of the two data streams need to be associated, e.g., based on timestamps. If a data point of one data stream has no valid association, then this point is discarded. While the data needs to be correctly associated, it does not have to be given in a causal time sequence. After the data points are synchronized, the recorded sequence is fixed and processed in Stage~1.

The known inputs are the essential states required to describe the location of a robot, commonly called \textit{core states}. For this work, the core states are defined by the position $\mathbf{p}_\mathsf{wi}$, rotation $\mathbf{R}_{(\mathbf{q}_\mathsf{wi})}$, and velocity $\mathbf{v}_\mathsf{wi}$ of an IMU-frame $\mathsf{i}$ with respect to the navigation world-frame $\mathsf{W}$.
Because we are including velocity sensor models, we also require the angular velocity information from the IMU~$\boldsymbol{\omega}_\mathsf{i}$. 
This sequence of core states 
\begin{equation}
        {}^t\mathbf{x}_\mathsf{c}=\left[ {}^t\mathbf{p}_\mathsf{wi}^\mathsf{T} ~ {}^t\mathbf{v}_\mathsf{wi}^\mathsf{T} ~ {}^t\mathbf{q}_\mathsf{wi}^\mathsf{T} ~ {}^t\boldsymbol{\omega}_\mathsf{i}^\mathsf{T} \right]^\mathsf{T} \in \mathbb{R}^{13\times n}
    \label{eq:core_states}
\end{equation}
is provided online by an IMU-driven EKF filter framework such as presented by \cite{Brommer2021_mars}. In scenarios where an IMU is not available but data such as angular velocity is required by a sensor model, this information would need to be estimated using available sensor data containing rotational information and e.g. numerical differentiation.

Stage~1, detailed in Section~\ref{sec:sensor_identification}, performs the main model identification based on a fixed given data sequence and a generalized model definition that adheres to the reference frame and state structure shown by Figure~\ref{fig:sensor_frames}.
According to this structure, a sensor can have an extrinsic calibration state describing a transformation of the sensor frame w.r.t. the IMU frame ($\mathbf{p}_\mathsf{is}$, and $\mathbf{R}_\mathsf{is}$).
The local deviation of the magnetic field is described as a Cartesian vector $\mathbf{m}_\mathsf{w}$ which does not need to align with any axis of the world frame and depends on the current geolocation~\cite{Brommer2020_mag}.

A dedicated sensor reference frame $\mathsf{R}$ is introduced because the information provided by a sensor entity does not need to align with the current global navigation world frame $\mathsf{W}$ of the estimator, thus requiring transformations $\mathbf{p}_\mathsf{rw}$ and $\mathbf{R}_\mathsf{rw}$. An example is given by a vision sensor that is initialized at a later point of a vehicle operation, and its reference is defined relative to the current location.

Sensor measurements can either be expressed w.r.t. the sensor frame $\mathsf{S}$ itself, such as the magnetic field $\mathbf{m}_\mathsf{s}$ and the body velocity $\mathbf{v}_\mathsf{s}$, or as position $\mathbf{p}_\mathsf{rs}$ and rotation $\mathbf{R}_\mathsf{rs}$ of the sensor frame w.r.t. the sensor reference frame $\mathsf{R}$, or as the inverse; the sensor reference frame w.r.t. the sensor frame~$\mathbf{p}_\mathsf{sr}$,~$\mathbf{R}_\mathsf{sr}$.
A special case is presented if the sensor measurement is the inverse $\mathbf{T}_\mathsf{rs}^{-1}$ (see Fig.~\ref{fig:sensor_frames}). In that case, the rotation $\mathbf{R}_\mathsf{rw}$ is a free parameter and not constrained by the sensor reference frame $\mathsf{S}$ nor the navigation world frame $\mathsf{W}$.
Thus, the transformations $\mathbf{p}_\mathsf{rw}$ and $\mathbf{R}_\mathsf{rw}$ are co-dependent and infinite valid solutions exist as long as $\mathbf{p}_\mathsf{rw}$ is chosen according to $\mathbf{R}_\mathsf{rw}$ to remain coherent.
Stage~2, described in Section~\ref{sec:sensor_state_calibration}, performs refinements and final choices on the calibration states after the general sensor model was chosen. 

\subsection{Sensor Identification}
\label{sec:sensor_identification}

The goal of this work is to identify the following sensor models commonly used in robot localization tasks. We use the notation of rotation matrices to express their state variables expressed in the tangent space e.g. $\exp(\boldsymbol{\omega}_\mathsf{wi}^\wedge) = \mathbf{R}_\mathsf{wi}$, to provide a concise description of the sensor models.
While the symbolic expressions for sensor states are the same for each sensor model in Figure~\ref{fig:sensor_frames}, each model has a dedicated state variable for the selection process according to Table~\ref{tab:state_sharing}. Only the sensor reference frame is defined as a common, shared, optimization state to provide an additional constraint and to restrain incorrect model selections while the reference frame converges.
Model states are given by
\begin{align}
\mathbf{x}_\mathsf{m} = [
    &\mathbf{p}_\mathsf{rw}^\mathsf{T}~
    \boldsymbol{\omega}_\mathsf{rw}^\mathsf{T}~
    {}^\mathsf{p}\mathbf{p}_\mathsf{is}^\mathsf{T}~
    {}^\mathsf{ip}\mathbf{p}_\mathsf{is}^\mathsf{T}~
    {}^\mathsf{ip}\boldsymbol{\omega}_\mathsf{is}^\mathsf{T}~
    {}^\mathsf{r}\boldsymbol{\omega}_\mathsf{is}^\mathsf{T}~ \nonumber \\
    &{}^\mathsf{ir}\boldsymbol{\omega}_\mathsf{is}^\mathsf{T}~ 
    {}^\mathsf{vel}\mathbf{p}_\mathsf{is}^\mathsf{T}~
    {}^\mathsf{bv}\mathbf{p}_\mathsf{is}^\mathsf{T}~
    {}^\mathsf{bv}\boldsymbol{\omega}_\mathsf{is}^\mathsf{T}~
    {}^\mathsf{mag}\boldsymbol{\omega}_\mathsf{is}^\mathsf{T}~
    \mathbf{m}_\mathsf{w}^\mathsf{T}
    ]^\mathsf{T} {\in \mathbb{R}^{36}}.
\end{align}
Let \mbox{$\mathbf{x} = \left[ {}^t\mathbf{x}_\mathsf{c}^\mathsf{T}~\mathbf{x}_\mathsf{m}^\mathsf{T} \right]^\mathsf{T}$} be the full state vector, containing the time sequence of the core states ${}^t\mathbf{x}_\mathsf{c}$ and the self-calibration states of all proposed sensor models ${}^t\mathbf{x}_\mathsf{m}$ for the following sensor model definitions:

{%
\setlength{\belowdisplayskip}{-3pt}%
\setlength{\abovedisplayskip}{-3pt}%
\begin{align}
    \intertext{Position}
    \mathbf{p}_\mathsf{rs}{(\mathbf{x})} &= \mathbf{p}_\mathsf{rw} + \mathbf{R}_\mathsf{rw} \left({}^t\mathbf{p}_\mathsf{wi} + {}^t\mathbf{R}_\mathsf{wi} ~ {}^\mathsf{p}\mathbf{p}_\mathsf{is}\right)
    \label{eq:position}\\
    \intertext{Inverse Position}
    \mathbf{p}_\mathsf{sr}{(\mathbf{x})} &= - {}^\mathsf{ip}\mathbf{R}_\mathsf{is}^\mathsf{T} ~
    \left(
    {}^t\mathbf{R}_\mathsf{wi}^\mathsf{T}
    \left(\mathbf{R}_\mathsf{rw}^\mathsf{T} ~ \mathbf{p}_\mathsf{rw} + {}^t\mathbf{p}_\mathsf{wi}
    \right)
    + {}^\mathsf{ip}\mathbf{p}_\mathsf{is}\right)
    \label{eq:inv_position}\\
    \intertext{Rotation}
    \mathbf{R}_\mathsf{rs}{(\mathbf{x})} &= \mathbf{R}_\mathsf{rw} ~ {}^t\mathbf{R}_\mathsf{wi} ~ {}^\mathsf{r}\mathbf{R}_\mathsf{is}
    \label{eq:rotation}\\
    \intertext{Inverse Rotation}
    \mathbf{R}_\mathsf{sr}{(\mathbf{x})} &= {}^\mathsf{ir}\mathbf{R}_\mathsf{is}^\mathsf{T} ~ {}^t\mathbf{R}_\mathsf{wi}^\mathsf{T} ~ \mathbf{R}_\mathsf{rw}^\mathsf{T}
    \label{eq:inv_rotation}\\
    \intertext{World Velocity}
    \mathbf{v}_\mathsf{ws}{(\mathbf{x})} &= {}^t\mathbf{v}_\mathsf{wi} + {}^t\mathbf{R}_\mathsf{wi} ~ {}^t\boldsymbol{\omega}_\mathsf{i}^\wedge ~ {}^\mathsf{vel}\mathbf{p}_\mathsf{is}
    \label{eq:velocity}\\
    \intertext{Body Velocity}
    \mathbf{v}_\mathsf{s}{(\mathbf{x})} &= {}^\mathsf{bv}\mathbf{R}_\mathsf{is}^\mathsf{T} ~ {}^t\mathbf{R}_\mathsf{wi}^\mathsf{T} ~ {}^t\mathbf{v}_\mathsf{wi} 
    + {}^\mathsf{bv}\mathbf{R}_\mathsf{is}^\mathsf{T} ~ {}^t\boldsymbol{\omega}_\mathsf{i}^\wedge ~ {}^\mathsf{bv}\mathbf{p}_\mathsf{is}
    \label{eq:body_velocity}\\
    \intertext{Magnetometer} 
    \mathbf{m}_\mathsf{s}{(\mathbf{x})} &= {}^\mathsf{mag}\mathbf{R}_{\mathsf{is}}^\mathsf{T} ~ {}^t\mathbf{R}_{\mathsf{wi}}^\mathsf{T} ~ \mathbf{m}_\mathsf{w}
    \label{eq:magnetometer}
\end{align}
    \vspace{2mm}
}
%

Table~\ref{tab:state_sharing} also shows the categorical type of states that are similar in purpose between sensor models. As a brief reference for future work; we evaluated a system set up where each sensor state was defined as a shared state between models i.e. ${}^\mathsf{ip}\mathbf{p}_{\mathsf{is}}$ and ${}^\mathsf{p}\mathbf{p}_{\mathsf{is}}$, the extrinsic calibration of the position and inverse position sensor was defined as one coupled state. The intent was to pose additional constraints in the form of coherence, but the results were limited in terms of robustness. This indicates that the system definition cannot be oversimplified, and individual states need to be granular while the system becomes more versatile.
However, the states for the reference frames $\mathbf{p}_\mathsf{rw}$ and $\mathbf{R}_\mathsf{rw}$ are shared which benefits robustness.
\begin{table}[htbp]
    \small
    \caption{Sensor model state association and sensor reference state sharing. Fields marked with $\mathsf{x}$ indicate a shared state, transformation entries represent a dedicated state for the model in the columns state category, and a free space indicates that a state is not required. 
    }%
    \vspace*{-2mm}
    \label{tab:state_sharing}
    \begin{center}
        \begin{tabular}{ | l | c | c | c | c |} 
          \hline
          \textbf{Model}   & $\mathbf{p}_{\mathsf{is}}$ & $\mathbf{R}_{\mathsf{is}}$ & $\mathbf{p}_\mathsf{rw}$ & $\mathbf{R}_\mathsf{rw}$ \\ 
          \hline
          Position         & ${}^\mathsf{p}\mathbf{p}_{\mathsf{is}}$ &   & $\mathbf{\mathsf{x}}$ & $\mathbf{\mathsf{x}}$ \\
          Inverse Position & ${}^\mathsf{ip}\mathbf{p}_{\mathsf{is}}$ & ${}^\mathsf{ip}\mathbf{R}_{\mathsf{is}}$ & $\mathbf{\mathsf{x}}$ & $\mathbf{\mathsf{x}}$ \\
          Rotation         &  &  ${}^\mathsf{r}\mathbf{R}_{\mathsf{is}}$ & & $\mathbf{\mathsf{x}}$ \\
          Inverse Rotation &  & ${}^\mathsf{ir}\mathbf{R}_{\mathsf{is}}$ & & $\mathbf{\mathsf{x}}$ \\
          World Velocity         & ${}^\mathsf{vel}\mathbf{p}_{\mathsf{is}}$ & &  &  \\
          Body Velocity    & ${}^\mathsf{bv}\mathbf{p}_{\mathsf{is}}$ & ${}^\mathsf{bv}\mathbf{R}_{\mathsf{is}}$ &  & \\
          Magnetometer     & & ${}^\mathsf{mag}\mathbf{R}_{\mathsf{is}}$ & & \\
          \hline
        \end{tabular}
    \end{center}
    \vspace{-5mm}
\end{table}%

Before continuing to the selection of states and models, we need to define the format of states and measurements. Translation states, as well as position $[m]$, velocity $[m/s]$, and magnetometer measurements are defined as direction vectors in the 2-Sphere~$S^2$. Because of this, the unit of the magnetometer measurements are not relevant for the detection process; units may be Tesla, 'counts' as provided by certain \acp{SoC}, or any.

Rotational states $\boldsymbol{\omega}$ within the sensor model definition are defined trough their tangent space value $\boldsymbol{\omega}$, and therefore, the parameters for a particular rotation are not unique. This does not pose issues in theory, but for real-world applications, this can result in numerical issues if the values for $\boldsymbol{\omega}$ are of higher magnitude.
Thus, given the axis-angle representation, we are limiting the magnitude of the angle for all rotational states in the range of $-\pi \le \boldsymbol{\omega} \le \pi$ by defining individual loss terms.
\begin{align}
    \mathcal{L}_\mathsf{rot}(\boldsymbol{\omega}) &= \begin{cases} 
    \lVert\boldsymbol{\omega}\rVert - \pi, & \text{if } \pi < \lVert\boldsymbol{\omega}\rVert\\
    0, & \text{otherwise} \\
    \end{cases} \\
    \text{with~} \frac{\boldsymbol{\omega}}{\lVert\boldsymbol{\omega}\rVert} \lVert\boldsymbol{\omega}\rVert &= \boldsymbol{\omega} = \log(\mathbf{R})^\vee,
\end{align}
$\mathbf{R}$ being a rotation matrix, $\frac{\boldsymbol{\omega}}{\lVert\boldsymbol{\omega}\rVert}$ the axis of rotation and $\lVert\boldsymbol{\omega}\rVert$ as the angle of rotation.

Our method utilizes soft Booleans $\mathcal{B} = [0, 1]$ for the selection of sensor models. Ideally, such soft Booleans can vary in a limited range during the optimization process $b_\mathsf{opt} \in \mathcal{B}$ and are categorized in a Boolean set in the final stage of the optimization $b_\mathsf{final} \in \{0, 1\}$. 
All selectors $b_\mathsf{opt}$ are valid candidates when the method is initialized. The soft Boolean selectors converge concurrently during the optimization while their value range is constrained by dedicated loss terms (see Fig.~\ref{fig:experiment_convergence} in the experiment section). Given ideal data, the selectors can converge to a single choice. Given realistic and imperfect data, the choice can become more ambiguous. Because of this, we define a decision and health metric to define $b_\mathsf{final}$ and to ensure that the correct model was selected.

As mentioned in the related work (Sec.~\ref{sec:related_work}), our goal is to not perform an iterative process such as, branch and bound, or randomization of Boolean code tables. The method described by this work performs an optimization and selection process on all states simultaneously in stage one.

Utilizing the described soft Booleans as selection operators allows for a grouping of multi-dimensional geometrical terms, which is required for the selection of sensor models within the over-determined system definition.
While LASSO promotes sparsity for individual state-values $\mathbf{v}_i \in \mathbb{R}^n$ as means of selections i.e., \mbox{$\min_{\boldsymbol{\alpha}}~\lVert \boldsymbol{\alpha} \rVert_1,~ \boldsymbol{\alpha} = [\mathbf{v}_1, \mathbf{v}_2, \ldots, \mathbf{v}_n]$}, the utilization of multiplicative soft Booleans ($b_{i} \in \mathcal{B}$) extends this approach and allows for the grouping and forming of multi-dimensional states. 
Let $\mathcal{L}_\mathsf{B}\left( \mathbf{\mathsf{b}} \right)$ be a collection of loss-terms defined by Equation~\eqref{eq:bool_loss} to enforce the properties of soft Booleans for the final system definition. This allows for multiplicative selections such as:
\begin{align}
    \boldsymbol{\beta} &= \left[b_1 \mathbf{v}_1,~b_2 \mathbf{v}_2,\ldots, b_n\mathbf{v}_n \right] \\
    \min_{\mathbf{b}}~y &= \sum_i \mathcal{L}_\mathsf{B}(b_i).
\end{align}
Where $b_i$ is the Boolean selector for the $\mathbf{v}_i$ state.
Please keep in mind that this example only shows the regularization as an example with respect to a full problem definition of a system.

We evaluated two concepts: a method for \textit{individual state selections}, which allows multiple selection operators to be true, and a \textit{direct selection approach}, which allows a \textit{one-out-of-many} single choice of pre-defined sensor models.

The concept of \textit{individual state selection} was limited in scalability and robustness for real-world scenarios. The corresponding system definition requires the augmentation of all sensor model states by individual soft Booleans and careful modeling of the inter-dependencies among proposed sensor models.
The main reason for the limited robustness is a sensitivity to the initial convergence phase. One of the main problems addressed by this work is that the selection of states and the convergence of state values should coincide. However, this approach was sensitive to signal attenuation and prone to cross-detection of states due to over-fitting, promoted by incorrect utilization of states.

Assuming a measurement was given by a position sensor. The concurring states of a magnetometer model would produce a signal component with a much lower magnitude compared to global position information. A magnetic field vector can freely rotate within the noise band of the position sensor signal and lead to the minimization of the system's residual, but results in an overall sensor model that is not meaningful.
This was also evaluated by using \ac{PSO} to circumvent such issues, but resulted in untraceable processing time and limited success for ill-conditioned scenarios.

The approach for \textit{pre-defined sensor model selection} did not fall short of the outlined issues and is presented in the following.
The method allows a robust selection of dedicated sensor models \mbox{(Eq.~\eqref{eq:position}-\eqref{eq:magnetometer})} as a whole and resolves the weaknesses of the approach for \textit{single state selection} which required individual loss terms for each selection operator, while the \textit{direct selection approach} requires a global set of regularizations to enforce the soft Boolean properties and achieve a \textit{one-out-of-many} selection.

A sensitivity analysis in Section~\ref{sec:sensitivity} has shown that the presented method produces zero false positives for data with high SNR, indicating that ideally, only one model can be selected for close to ideal data.

While this is true, given simulated and noise-free data with a trajectory that ensures the observability of states, an application with real-world data requires additional steps to achieve a robust result. The first step is to normalize the unknown measurement as well as the individual components of the over-determined system model to the value range $[0,1]$.
This normalization mitigates errors due to over-fitting and sensitivity to signal attenuation, which was shown as an example in the previous section. In particular, the magnetometer sensor model (Eq.\eqref{eq:magnetometer}) and especially the ubiquity of its state $\mathbf{m}_\mathsf{w}$ can freely contribute to the minimization of noise-afflicted residuals.
Measurement and model normalization addresses this issue while also improving resilience to noise (see~Sec.~\ref{sec:experiments}).

However, noise-afflicted data and uncertainties of the given core states can still prevent a definite single model selection since a number of sensor models can contribute to a solution partially. Thus, a set of constraints, operating on the selection operators $\mathbf{\mathsf{b}}$, are defined to enforce the convergence to a single selection. Given $b_* \in \mathcal{B}$, $\mathbf{\mathsf{b}} \in \mathcal{B}^n$, the following loss terms can be combined as 
\begin{equation}
\mathcal{L}_\mathsf{B}\left( \mathbf{\mathsf{b}} \right) = \mathcal{L}_\mathsf{norm}\left( \mathbf{\mathsf{b}} \right) + \mathcal{L}_\mathsf{pos}\left( \mathbf{\mathsf{b}} \right) + \mathcal{L}_\mathsf{std}\left( \mathbf{\mathsf{b}} \right). \label{eq:bool_loss}
\end{equation}
%
First, enforce that $\mathbf{\mathsf{b}}$ is constrained to a single overall contribution of models 
\begin{equation}
\mathcal{L}_\mathsf{norm} = \left| \left\lVert \mathbf{\mathsf{b}} \right\rVert_1 - 1 \right|_2 , \label{eq:def_loss_norm}
\end{equation}
which does not over-constrain the system and allows for a free convergence of state-values and selectors because a mixture of models is still possible.
A second constraint enforces that all selectors are positive by posing a penalty for negative selectors
\begin{equation}
\mathcal{L}_\mathsf{pos} = \left\lVert \{ b \mid b \in \mathbf{\mathsf{b}},\, b < 0 \} \right\rVert_2.\label{eq:def_loss_pos}
\vspace{-2mm}
\end{equation}
This constraint is crucial for the rotation models (Eq.~\ref{eq:rotation}~and~\ref{eq:inv_rotation}) because the measurement of a rotation signal and the model definition provide rotational information expressed in the tangent space. Since, \mbox{$\log(\mathbf{R}^\mathsf{T})^\vee \equiv -\boldsymbol{\omega}$}, the definition of the rotation and inverse rotation model can interfere and possibly cancel each other out while, at the same time, reaching a valid decision for the correct module.
A valid solution from the system point of view could otherwise be a selection with \mbox{$\left([{b}_{\mathbf{R}_\mathsf{sr}} = 0.5, {b}_{\mathbf{R}_\mathsf{rs}} = -0.5, {b}_{\mathbf{p}_\mathsf{sr}} = 1] \right)$}, leading to ambiguities for the decision-making process.

While the normalization already improves issues due to signal attenuation, an additional improvement can be made by enforcing single vector module entities to $\left\lVert\mathbf{v}\right\rVert_2=1$. In the presented case, this only concerns the magnetometer model and renders the constraint to
\begin{equation}
\mathcal{L}_\mathsf{vec} = \lVert \lVert \mathbf{\mathbf{m}_\mathsf{s}} \rVert_2 - 1 \rVert_2.\label{eq:def_loss_vec}
\end{equation}
Single vector entities can highly contribute to over-fitting and false positives, especially during the transient phase.
For this reason, this loss term is squared in the final system definition (Eq.~\ref{eq:sys_loss_group}) to provide steeper gradients for faster convergence to fully avoid possible incorrect convergence of the overall system. 

Finally, to enforce a \textit{single} decision, we utilize the regularization introduced by Theorem~\ref{lemma:std_reg} in conjunction with $\mathcal{L}_\mathsf{norm}$, which ensures $\left\lVert \mathbf{b_\mathsf{vec}} \right\rVert_1 = 1$ as
\begin{equation}
\mathcal{L}_\mathsf{std} = \left\lVert std(\mathbf{b}_\mathsf{vec}) - \frac{1}{\sqrt{N}} \right\rVert_2 , \label{eq:def_loss_std}
\end{equation}
with $std(\mathbf{x}) = \sqrt{ \frac{1}{N-1}  \sum_{i}(x_i - \mu)^2}$ and $N$ as dimension of~$\mathbf{\mathsf{b}}$.

The reasoning behind the loss-term $\mathcal{L}_\mathsf{std}$ and the corresponding theorem~\ref{lemma:std_reg} originated based on the property of maximum variation for a vector \mbox{$\mathbf{x} \in \mathbb{R}^N$} as follows: If we pose the constraints $\mathcal{L}_\mathsf{norm}$ and $\mathcal{L}_\mathsf{pos}$ onto a decision vector $\mathbf{x}$ with soft Boolean and non-random elements $x_i$, then the standard deviation of this vector has an upper limit, equivalent to the maximum variation of $\mathbf{x}$. The proof shows, that the specific maximum variation of vector $\mathbf{x}$ is $\frac{1}{\sqrt{N}}$ and can only be reached if the vector has a single entry of one and zeros otherwise.
This property is essential for loss-term $\mathcal{L}_\mathsf{std}$ and the presented one-out-of-many decision making approach which enforces a single selection.
\vspace{2mm}

\begin{theorem}
Let $\mathbf{x} \in \mathbb{R}^N$ be a vector.
Let $\mathit{std}(\mathbf{x})$ represent the standard deviation of $\mathbf{x}$. Assume $0 \leq x_i \leq 1 \; \forall\; i = 1,\dots,N$ of vector $\mathbf{x}$, $\lVert \mathbf{x} \rVert_1 = 1$, and $\text{std}(\mathbf{x}) = \frac{1}{\sqrt{N}}$. Let $M$ be the number of non-zero entries $x_i$ of $\mathbf{x}$. Then, there are only $N$ constellations for $\mathbf{x}$ with $M = 1$, and zero constellations with $M > 1$.
\label{lemma:std_reg}
\end{theorem}

\begin{proof}
\hfill
\begin{compactenum}[\hspace{0.25cm} 1. ]
\item We proof by contradiction and therefore assume\\$\exists\; M > 1 \;|\; 0 < x_i < 1 \; \forall \; i=1,\dots,M$.

\item Given the hypothesis $std(\mathbf{x}) = \frac{1}{\sqrt{N}}$ and $\lVert \mathbf{x} \rVert_1 = 1$, we~get
\begin{align*}
std(\mathbf{x}) = \frac{1}{\sqrt{N}} &= \sqrt{\frac{1}{N-1} \sum_i^N(x_i-\mu)^2}, \\
\frac{1}{N} &= \frac{1}{N-1} \sum_i^N(x_i-\mu)^2,\\
\frac{N-1}{N} &= \sum_i^N(x_i-\mu)^2
\end{align*}

\item where $\sum_i^N(x_i-\mu)^2$ can be split in two sets,\\ $S_0 = \{i \mid x_i = 0\},~S_{\alpha} = \{i \mid 0 < x_i < 1\}$ of vector $\mathbf{x}$
\begin{align*}
\frac{N-1}{N} &= \sum^M_{i \in S_\alpha}(x_i - \mu)^2 + \sum^{N-M}_{i \in S_0} \mu^2. \\
\mu = \frac{1}{N} \sum^N_{i} x_i &= \frac{1}{N} \left(\sum^M_{i \in S\alpha} x_i + \sum^{N-M}_{i\in S_0} 0\right) = \frac{1}{N}.
\end{align*}
\item Substituting for $\mu$, we can show that:
\begin{align*}
\frac{N-1}{N} &= \sum^M_{i \in S_\alpha}\left(x_i - \mu \right)^2 + \sum^{N-M}_{i \in S_0} \mu^2,\\
&= \sum^M_{i \in S_\alpha}\left(x_i - \frac{1}{N}\right)^2 + \sum^{N-M}_{i\in S_0} \left(\frac{1}{N}\right)^2\\
&= \sum^M_{i \in S_\alpha} \left(x_i^2 - 2 x_i \frac{1}{N} + \frac{1}{N^2} \right) + \frac{N-M}{N^2} \\
 &= \sum^M_{i \in S_\alpha} \left(x_i^2 \right) - \frac{2}{N}  \sum^{{M}}_{i \in S_\alpha} x_i + \frac{M}{N^2} + \frac{N-M}{N^2} \\
 &= \sum^M_{i \in S_\alpha} \left(x_i^2 \right) - \frac{2}{N}  \sum^{{M}}_{i \in S_\alpha} x_i + \frac{1}{N},
\end{align*}
\item where $\sum^{{M}}_i x_i = \lVert 1 \rVert_1 = 1$ was used in the last equation. It follows that
\begin{align*}
\frac{N-1}{N} &= \sum^M_{i \in S_\alpha} \left(x_i^2 \right) - \frac{1}{N},\\
1 - \frac{1}{N} &= \sum^M_{i \in S_\alpha} \left(x_i^2\right) - \frac{1}{N},\\
\sum^M_{i \in S_\alpha}(x_i^2) &= 1.
\end{align*}
\item The last equality is true if and only if $M=1$, proving the theorem.
\end{compactenum}
\end{proof}

An example can be given for the two dimensional case:

\begin{compactenum}[\hspace{0.25cm} 1. ]
\item Given the following equality of the constraints \mbox{$\sum_i^N(x_i^2) = \lVert \mathbf{x} \rVert_1 = 1$} and an example with $N=2$
\begin{align*}
\lVert \mathbf{x} \rVert_1 &= x_1 + x_2 = 1\\ 
x_1 &= 1 - x_2.
\end{align*}
\item Substituting for $x_1$, and solving the polynomial for $x_2$ we can show that
\begin{align*}
\sum_i^N(x_i^2) &= x_1^2 + x_2^2 = 1 \\
 &= (1 - x_2)^2 + x_2^2 = 1,\\
0 &= x_2^2 - x_2 ~ \Rightarrow x_{2_{1,2}}=\{ 0,1\}.
\end{align*}
\item Solving for $x_1$ given $\sum_i^N(x_i^2)$ and the solutions for $x_2$
\begin{align*}
x_{1_{1,2}} = \{ 1,0\} ~\text{given}~ x_{2_{1,2}}=\{ 0,1\},
\end{align*}
\item shows mutual exclusive boolean selectors and supports the proof that constraints are only fulfilled for $M=1$.
\end{compactenum}
%

\vspace{2mm}
A visual proof that the sets of values for the given constraints are only inclusive where the conditions mentioned above are met, is visualized by Figure~\ref{fig:std_loss}.
Experimental results, which show how this loss term improves the soft Boolean convergence, is presented by Figure~\ref{fig:experiment_convergence} (Sec.~\ref{sec:experiments}).\\

The full and final approach can be addressed with common non-linear solvers such as the Levenberg–Marquardt algorithm and numerically approximated Jacobians. This is beneficial due to the high variability of the system definition and its consecutive Jacobians.

\noindent The system definition of this approach is defined by
\begin{align}
f_\mathsf{sys}(\mathbf{x}_\mathsf{c},\mathbf{x}_\mathsf{m},\mathbf{\mathsf{b}}) =& 
    {b}_{\mathbf{p}_\mathsf{sr}} \, \mathbf{p}_\mathsf{sr}(\mathbf{x}) +
    {b}_{\mathbf{p}_\mathsf{rs}} \, \mathbf{p}_\mathsf{rs}(\mathbf{x})
    + {b}_{\mathbf{R}_\mathsf{rs}} \,{\boldsymbol{\omega}_\mathsf{rs}}(\mathbf{x}) \label{eq:sys_group_modules} \\  
    + {b}_{\mathbf{R}_\mathsf{sr}} \,{\boldsymbol{\omega}_\mathsf{sr}}(\mathbf{x}) +& {b}_{\mathbf{v}_\mathsf{ws}} \, \mathbf{v}_\mathsf{ws}(\mathbf{x}) +
    {b}_{\mathbf{v}_\mathsf{s}} \, \mathbf{v}_\mathsf{s}(\mathbf{x}) +
    {b}_{\mathbf{m}_\mathsf{s}} \, \mathbf{m}_\mathsf{s}(\mathbf{x}) \notag\\
\mathcal{L}_\mathsf{loss}(\mathbf{\mathsf{b}}) =& \lambda_\mathsf{n} \, \mathcal{L}_\mathsf{norm}(\mathbf{\mathsf{b}}) 
 + \lambda_\mathsf{p} \, \mathcal{L}_\mathsf{pos}(\mathbf{\mathsf{b}}) \label{eq:sys_loss_group}\\
+& \lambda_\mathsf{s} \, \mathcal{L}_\mathsf{std}(\mathbf{\mathsf{b}}) + 
\lambda_\mathsf{v} \, \mathcal{L}_\mathsf{vec}(\mathbf{\mathsf{b}})^2 \notag
\end{align}
\begin{equation}
\min_{{\mathbf{x}_\mathsf{m}, \mathbf{\mathsf{b}}}}~y =
\frac{1}{2}  \frac{
\sum_{i} \left\lVert f_\mathsf{sys}({\mathbf{x}_\mathsf{c_i},\mathbf{x}_\mathsf{m_i},\mathbf{\mathsf{b}}}) - {}^t\mathbf{S}_i\right\rVert
}{
\sqrt{n_\mathsf{samples}}
} + \mathcal{L}_\mathsf{loss}({\mathbf{\mathsf{b}}})
\end{equation}%
%
with the over-determined system definition $f_\mathsf{sys}$, previously introduced constraints $\mathcal{L}_\mathsf{loss}$, individual weights~$\lambda_*$, $\mathbf{S}$ as the unknown sensor signal, and $n_\mathsf{samples}$ as the number of sample points to render the system error invariant to the length of a dataset.

In this case, the norm of $\mathbb{R}^{3\times n}$ is the Frobenious norm. By normalizing the Frobenious norm based on $\sqrt{n_\mathsf{samples}}$, we obtain the RMSE, which is agnostic to the number of samples.

Since the system components are normalized for this approach, the state values of the selected module are unscaled and need to be refined, which is the purpose of stage~2~(Sec.~\ref{sec:sensor_state_calibration}).

An additional note on module similarities:
By examining Eq.~\eqref{eq:mag_detail} and \eqref{eq:rot_detail}, it can be seen that the result for both expressions is a vector in $\mathbb{R}^{3}$.
An issue can arise if the rotations defined by \eqref{eq:rot_detail} do not fully minimize the error and express $\log(\mathbf{R}_\mathsf{sr})^\vee = \boldsymbol{\omega}_\mathsf{sr}$ incorrectly.
For this scenario, a multiplicative vector could also minimize the error and contribute to the solution due to overfitting or over-determination. In this case, Eq.~\eqref{eq:mag_detail} can double as a solution for a rotation sensor module.
\begin{align}
\mathbf{m}_\mathsf{s} &= \mathbf{R}_{\mathsf{is}}^\mathsf{T} ~ {}^t\mathbf{R}_{\mathsf{wi}}^\mathsf{T} ~ \mathbf{m}_\mathsf{w} \label{eq:mag_detail} \\
\log(\mathbf{R}_\mathsf{sr})^\vee &= \log(\mathbf{R}_\mathsf{is}^\mathsf{T} ~ {}^t\mathbf{R}_\mathsf{wi}^\mathsf{T} ~ \mathbf{R}_\mathsf{rw}^\mathsf{T})^\vee \label{eq:rot_detail}
\end{align}
This issue can be circumvented by enforcing additional constraints on $\mathbf{m}_\mathsf{w}$. This is the reason why we decided to fully isolate rotation signal identification from all other sensor module expressions.
Besides the issue described above, the main reason for this decision is that rotational signals need to be expressed in the tangent space and require additional conversion to $\mathbb{R}^{3}$ before the detection in stage~1 can be performed.
The detection and refinement of calibration states for rotational information are still of interest to define the nature of the measurement expression (normal or inverse) and the calibration states.

Defining the weights for loss terms is often a heuristic approach. However, the weights for this system setup (Eq.\eqref{eq:sys_loss_group}) are defined such that all initial loss terms result in approximately the same range of magnitude. Weights for the presented setups are $\lambda_\mathsf{p} = 200, \lambda_\mathsf{n} = 50, \lambda_\mathsf{v} = 100, \lambda_\mathsf{s} = 20$, and the development of weighted loss terms throughout the optimization process is shown by Figure~\ref{fig:experiment_convergence}.

Finally, the decision-making process is straightforward. Given the final result of the selectors, $\mathbf{b_\mathsf{vec}}$, the soft Boolean with the highest magnitude is selected to determine a candidate for the model choice.
Since we know that the problem definition is not convex, we need to ensure that the choice is not a false positive before proceeding to the next stage.
Thus, a health metric, based on the distance of the soft Boolean with the highest and second highest magnitude, together with the remaining loss terms that enforce the soft Boolean properties, are used to identify false positives.
This metric is detailed and validated in the experiment Section~\ref{sec:healt_metric}.

Overall, this approach reaches a final decision while preventing false positive decisions due to noise and uncertainties of the provided input signals and module cross-detections due to signal attenuation. The scalability of this approach is also improved with respect to the single-state selection because of minimized system inter-dependencies.

\begin{figure}[!tb]
    \centering
    \includegraphics[width=0.7\linewidth, trim={0 0mm 0 0mm},clip]{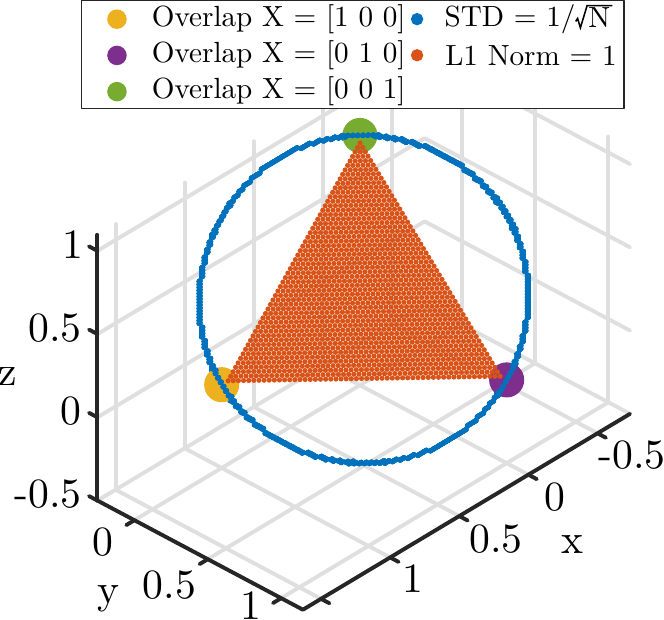}
    \caption{Visual proof of Theorem~\ref{lemma:std_reg}}
    \label{fig:std_loss}
    \vspace{-6mm}
\end{figure}

\subsection{Sensor State Calibration}
\label{sec:sensor_state_calibration}

This section details the process of stage~2 (see Fig.~\ref{fig:method_outline}) concerning the estimation of sensor state calibrations and the necessity of reference frames for dedicated sensor modules. In general, the same sensor model definitions that are used for stage~1 are used for stage~2, with the difference that only the previously detected sensor model is selected and processed.
While the unknown measurement input and the estimated measurement were normalized in stage~1, e.g., for noise resilience, stage~2 requires the raw system information for the determination of calibration states.

Since the problem definition is non-convex, the initialization point of optimization parameters is vital for correct convergence. We performed tests that showed that a naive initialization of these states could lead to incorrect results.
A more robust initialization and convergence is given if the states are initialized with the final parameters from stage~1. These parameters are obtained with the normalized system error definition in stage~1, possibly leading to higher resilience to local minima due to lower gradients in the optimization landscape. Experiment Section~\ref{sec:ref_frame_determination} shows corresponding results.

The system error definition for rotations differs for stage~1 and stage~2. As introduced earlier, the rotations for the calculation of the system error are expressed in the tangent space. In this space, the expression of rotations is not unique and reoccurring.
Since these parameters are not normalized in stage~2, it needs to be assured that the repetitiveness of values expressing the same rotations is consistent and comparable for the system input and the estimated measurement. Thus, the error between two series of rotations $\boldsymbol{\omega}_\mathsf{a}$ and $\boldsymbol{\omega}_\mathsf{b}$ for rotation models in stage~2 is defined as follows:
\begin{equation}
   \epsilon_\mathsf{rot} =  \frac{ \sum\limits_{i} \left\lVert \log(\exp(\boldsymbol{\omega}_{{A_i}}^\wedge) ~ \exp(\boldsymbol{\omega}_{{B_i}}^\wedge)^\transpose)^\vee \right\rVert_2}{N},
\end{equation}
this ensures that rotations are unique trough rotation matrices in $SO\left(3\right)$ when determining the error.

It also needs to be noted that the states of the models need to be observable. This concerns the reference frames for the rotation and the inverse position sensors.
Using Figure~\ref{fig:sensor_frames} as a reference; for rotation modules, the reference frame rotation $\mathbf{R}_\mathsf{rw}$ is unobservable and needs to be fixed in order to obtain a valid calibration for $\mathbf{R}_\mathsf{is}$.
For the inverse position module, it can be seen that the sensor reference frame~\rframe{R} is not constrained by a rotation. Thus, the translation and rotation of the reference frame are codependent, and there are multiple combinations for $\mathbf{p}_\mathsf{rw}$ and $\mathbf{R}_\mathsf{rw}$ that are expressing a valid solution. 

The second step within stage~2 is the reference frame determination. Until this point, the model detection and estimation of initial calibration states assumed that sensors require a reference frame. The last step in the model and state identification is to determine if a reference frame is required.
This is important information for consecutive steps where the detected sensor model is parameterized and used by a recursive filter, allowing for sensor self-calibration. If a reference frame is not required, it is beneficial not to include it.

The reference frame is defined by $\mathbf{p}_\mathsf{rw}$ and $\mathbf{R}_\mathsf{rw}$. We use the LASSO approach and the L1-Norm as a penalty for the reference frame. The following loss is added to the optimization process for reference frame detection, while the sensor to IMU calibration is also subject to optimization,
\begin{equation}
   \mathcal{L}_\mathsf{ref} =  \left\lVert \mathbf{p}_\mathsf{rw} \right\rVert_1 +  \left\lVert \boldsymbol{\omega}_\mathsf{rw} \right\rVert_1.
\end{equation}
$\mathcal{L}_\mathsf{ref}$ represents a non-minimizable loss term if a reference frame is required. Estimating the calibration state itself is a property of the individual sensor model definition in Equation~\eqref{eq:sys_group_modules}.
This process allows for clear decisions on the existence of reference frames based on the remaining magnitude of the reference states. Experimental results are shown in Section~\ref{sec:ref_frame_determination}.
%

\section{Experiments}
\label{sec:experiments}
This section investigates the reliability and properties of the presented method. Initial experiments are performed on ten simulated Lissajous figures, perturbed with various noise settings to analyze the sensitivity of the approach (Sec.~\ref{sec:sensitivity}) and reliability of the introduced health metric (Sec.~\ref{sec:healt_metric}) on 1750 case scenarios for statistical significance.
In addition, the approach is applied within a real-world scenario to prove its applicability (Sec.~\ref{sec:experiment_real_world_uav} and \ref{sec:experiment_real_world_ground_vehicle}).

The simulated trajectories have the following stats: Range in position~$0\text{-}\SI{3.5}{\meter}$, max. velocity~\SI{3.5}{\meter\per\second}, roll and pitch angles within~$\pm\SI{30}{\degree}$, yaw~$\pm\SI{30}{\degree}$, lateral max. acceleration~$\SI{1.7}{\meter\per\second\squared}$, and max. angular velocity~\SI{3.2}{\radian\per\second}.
Experiments for the sensitivity analysis add specific noise to individual system inputs. Unless stated otherwise, the following noise parameters are used for a realistic setup. 
For the general approach, we assume that core state information is provided by an estimator. However, for the simulated experiments, we are using ground~truth data with additive noise to provide core states with relatively noisy data. This ensures repeatability and consistency throughout all simulated cases.
Please keep in mind that these are \mbox{1-sigma} values. The corresponding \mbox{3-sigma}, especially for the rotation, is intended to prove that this method also works with high/sufficient noise:

\textit{System input:}
Position~$\SI{0.1}{\meter}$,
velocity~$\SI{0.05}{\meter\per\second}$,
orientation~$\SI{2}{\degree}$,
angular velocity~$0.014~deg/s/\sqrt{hz}$ (continuous) as for the BMI055 MEMS IMU, used e.g. by the PX4 autopilot.
\textit{Measurements:}
Position~$\SI{0.3}{\meter}$,
Rotation~$\SI{3}{\degree}$ axis angle rotation,
Velocity~$\SI{0.05}{\meter\per\second}$,
Magnetometer~$\SI{6}{\degree}$ axis angle rotation.

\subsection{Sensitivity Analysis and Validation}
\label{sec:sensitivity}

This section details a sensitivity analysis of the proposed method and the system definition to gain insights into robustness and noise resilience. The presented approach is designed to identify a single sensor modality at a time and therefore, the following evaluation is performed for individual sensor models.
The failure points and reported false selections are direct outcomes of stage~1 and are not evaluated in conjunction with the presented health metric~(Sec.\ref{sec:healt_metric}).
In the following, Tables~\ref{tab:sensitivity_meas_noise}, \ref{tab:sensitivity_snr_collective}, and \ref{tab:sensitivity_snr_individual} show the results for three different sensitivity analysis aspects by evaluating successful and failed detections with respect to different noise levels. The setup and objective for each experiment is detailed in the following. However, all experiments use the same ten individual trajectories and the tables detail the number of false detections out of the ten trails each.
It is expected that low noise profiles cause no false detections and, e.g., unrealistically high noise perturbations cause false detections. High noise scenarios can also be interpreted as sensor failure and data with a significant number of outliers. In such cases, it is important that failures are detected reliably. A different data segment can then be used for re-evaluation.

This analysis follows two approaches, the first evaluation is done with reasonable noise for the system inputs (core states) and gradually increasing sensor measurement noises.
This relates the analysis to real-world applications and intuitively interpretable results. Table~\ref{tab:sensitivity_meas_noise} shows the results when increasing the noise levels until a respective sensor type fails.
It can be seen that the position and rotation detections hardly fail despite the comparatively high noise levels. This is due to the characteristics of the measurement types, the loss terms that prevent over-fitting and cross-detection, as well as reasonable uncertainty of the core states. As a reference to the health metric in the next section, none of the position and rotation selections, which have been correctly selected, were close to the decision threshold.

\begin{table}[htbp]
    \footnotesize
    \ra{0.5}
    \caption{Variation of the sensor measurement noise given a reasonable uncertainty of the system states (Sec.\ref{sec:experiments}).
    The table shows the number of false decisions for 10 different trajectories.
    High noise values are intended to result in false choices to show the sensitivity of the system.}%
    \label{tab:sensitivity_meas_noise}
    \vspace{-2mm}
    \begin{center}
        \begin{tabular}{ l | P{12px} P{12px} P{12px} P{12px}  }
        \toprule 
        \textbf{Sensor \& Noise} & \multicolumn{4}{c}{\footnotesize False selections in 10 trails}\\
        \midrule
        \textbf{Noise} $[m]$  & \textbf{3} & \textbf{6} & \textbf{8} & \textbf{10}  \\
        Position     & 0 & 0 & 0 & \cellcolor{t2}2  \\
        Inv Position & 0 & 0 & 0 & 0  \\
         \midrule
         \midrule
        \textbf{Noise} axis angle $[deg]$ & \textbf{10} & \textbf{45} & \textbf{90} & \textbf{135} \\
        Rotation      & 0 & 0 & 0 & 0 \\
        Inv Rotation  & 0 & 0 & 0 & 0 \\
         \midrule
         \midrule
        \textbf{Noise} $[m/s]$ & \textbf{0.5} & \textbf{1.5} & \textbf{2} & \textbf{3}  \\
        World Velocity      & 0 & 0 & \cellcolor{t2}2 & \cellcolor{t8}8  \\
        Body Velocity & 0 & 0 & \cellcolor{t3}3 & \cellcolor{t10}10  \\
          \midrule
          \midrule
        \textbf{Noise} axis angle $[deg]$ & \textbf{10} & \textbf{45} & \textbf{90} & \textbf{135} \\
        Magnetometer & 0 & 0 & 0 & \cellcolor{t3}3 \\
        \bottomrule
        \end{tabular}
    \end{center}
    \vspace{-5mm}
\end{table}

The second approach relates the input noise to the variation of the trajectory and thus provides a more comparable insight into the sensitivity of the system.
For a fair validation, we need to establish a common ground for the comparison of various trajectories and corresponding changes for the input data.
For this, we utilize a definition of the \textit{signal-to-noise ratio} (SNR).
Since the SNR is a local ratio between the power of a signal $m$ and the power of its noise $\sigma$, the standard SNR metric $\text{SNR} = \frac{m}{\sigma}$ would be impractical given robotics applications because the power of the noise would need to change with the power of the signal. 

Sensors in robotics usually have a magnitude-independent noise. Thus, we are using a modified definition of the SNR and use the standard deviation $\lambda$ of the trajectory (system input) as a metric for the variation of the signal and the power of the noise $\text{SNR} = \frac{\lambda(\mathbf{x})}{\sigma_n}$ with $\lambda(\mathbf{x}) = \sqrt{ \frac{1}{N}  \sum_{i}(x_i - \mu)^2}$. This allows the generation of noise-afflicted data with the same SNR for multiple test trajectories. The noise for each trajectory will be different according to this definition.

While the noise for translations and vectors is additive, the noise for rotations and the magnetic vector $\mathbf{m}_\mathsf{w}$ are generated using the definition of the axis angle rotation based on the SNR $\alpha$ which, in the presented case, affects each axis equally
\begin{align}
\alpha &= \lVert \boldsymbol{\omega} \rVert = \sqrt{\omega_x^2+\omega_y^2+\omega_z^2},  \\
\alpha^2 &= 3~\omega_i^2 \text{~for equal effect on each axis},\\
\omega_i &=\frac{\alpha}{\sqrt{3}}.
\end{align}

Table~\ref{tab:sensitivity_snr_collective} shows the results of a setup in which all system inputs are noise afflicted with the same SNR to the point of failure, done for all sensor model inputs individually. The results show that most model detection fails at $\textrm{SNR}=0.5$, which determines that the noise floor is twice the magnitude of the variation of the input signal.

\begin{table}[htbp]
    \footnotesize
    \ra{0.5}
    \caption{Variation of the trajectory dependent SNR for all system inputs, the known states, and unknown sensor measurements, for a fair comparison across trajectories.}%
    \vspace{-3mm}
    \label{tab:sensitivity_snr_collective}
    \begin{center}
        \begin{tabular}{ l | P{8px} P{8px} P{8px} P{8px} P{8px}}
        \toprule 
        \textbf{SNR} & 2 & 1.5 & 1 & 0.5 & 0.4 \\
        \toprule
        \textbf{Sensor} & \multicolumn{5}{c}{ \footnotesize False selections in 10 Trails }\\
        \midrule
        Position     & 0 & 0 & 0 & 0 & 0 \\
        Inv Position & 0 & 0 & 0 & \cellcolor{t1}1 & \cellcolor{t1}1  \\
        Rotation     & 0 & 0 & 0 & \cellcolor{t2}2 & \cellcolor{t5}5  \\
        Inv Rotation & 0 & 0 & 0 & \cellcolor{t2}2 & \cellcolor{t2}2  \\
        World Velocity     & 0 & 0 & \cellcolor{t2} 2 & \cellcolor{t10} 10 & \cellcolor{t10}10  \\
        Body Velocity& 0 & 0 & \cellcolor{t2} 2 & \cellcolor{t9} 9 & \cellcolor{t7}7 \\
        Magnetometer & 0 & 0 & 0 & \cellcolor{t3}3 &\cellcolor{t3}3  \\
         \bottomrule
        \end{tabular}
    \end{center}
    \vspace{-4mm}
\end{table}

Finally, Table~\ref{tab:sensitivity_snr_individual} shows a more elaborate analysis with noise perturbations and all system parameters individually over a range of SNRs and all individual sensor signals as inputs.
It can be seen that the majority of detections fail if the noise for the measurement itself is high (SNR $<0.5$, column $\mathsf{m}$). Furthermore, the magnetometer and body velocity detection is sensitive to high noise on core state rotation (SNR $<0.5$, column $\mathsf{r}$), and the velocity is slightly sensitive to high core state velocity noise (SNR $<0.4$, column $\mathsf{r}$).

All tests concerning sensitivity show that position and rotation detections are generally more robust, which is due to their characteristics, as described in the beginning.

\begin{table}[htbp]
    \scriptsize
    \ra{0.3}
    \setlength\tabcolsep{3.1pt}
    \caption{Individual SNR perturbations on all system inputs while retaining respective other inputs at reasonable uncertainty.
    This experiment aims to visualize the sensitivity of the posed sensor models (rows) towards noise perturbations on core states (columns). System inputs: Position $p$, velocity $v$, rotation $r$, and angular velocity $\omega$ as well as measurement input $m$.}%
    \label{tab:sensitivity_snr_individual}
    \vspace{-1mm}
    \begin{center}
        \begin{tabular}{ l | P{7px}P{7px}P{7px}P{7px}P{7px}  |P{7px}P{7px}P{7px}P{7px}P{7px} | P{7px}P{7px}P{7px}P{7px}P{7px} }
        \textbf{SNR} &  \multicolumn{5}{c|}{1} & \multicolumn{5}{c|}{0.5}& \multicolumn{5}{c}{0.4} \\
        \toprule
        \textbf{Input} & $p$ & $v$ & $r$ & $\omega$ & $m$  & $p$ & $v$ & $r$ & $\omega$ & $m$  & $p$ & $v$ & $r$ & $\omega$ & $m$  \\
        \midrule
        \textbf{Sensor} & \multicolumn{15}{c}{Number of false selections in 10 trails}\\
        \midrule
        Position         & 0 & 0 & 0 & 0 & 0  & 0 & 0 & 0 & 0 & 0  & 0 & 0 & 0 & 0 & 0 \\
        Inv Position     & 0 & 0 & 0 & 0 & 0  & 0 & 0 & 0 & 0 & 0  & 0 & 0 & 0 & 0 & 0 \\
        Rotation         & 0 & 0 & 0 & 0 & 0  & 0 & 0 & 0 & 0 & \cellcolor{t2}2  & 0 & 0 & 0 & 0 & 0 \\
        Inv Rotation     & 0 & 0 & 0 & 0 & 0  & 0 & 0 & 0 & 0 & 0  & 0 & 0 & 0 & 0 & \cellcolor{t1}1 \\
        World Velocity         & 0 & 0 & 0 & 0 & 0  & 0 & \cellcolor{t3}3 & 0 & 0 & \cellcolor{t3}3  & 0 & \cellcolor{t6}6 & 0 & 0 & \cellcolor{t9}9 \\
        Body Velocity    & 0 & 0 & 0 & 0 & 0  & 0 & \cellcolor{t1}1 & \cellcolor{t7}7 & 0 & \cellcolor{t2}2  & 0 & \cellcolor{t2}2 & \cellcolor{t8}8 & 0 & \cellcolor{t6}6 \\
        Magnetometer     & 0 & 0 & 0 & 0 & 0  & 0 & 0 & \cellcolor{t3}3 & 0 & 0  & 0 & 0 & \cellcolor{t7}7 & 0 & 0 \\
         \bottomrule
        \end{tabular}
    \end{center}
    \vspace{-3mm}
\end{table}

Both approaches serve a specific purpose; the first approach applies to real-world scenarios and to which point a noise-afflicted sensor signal can be detected, given that the system's state is reasonably accurate. The second approach uses comparable noise settings for all trajectories and individually challenges the overall system to the point of failure and reveals the sensitivity of individual components toward one another.

As an additional validation element, Figure~\ref{fig:experiment_convergence} visualizes the effect of the STD loss term and the convergence of the loss functions.

In the following, we are introducing two examples for the velocity sensor identification from the larger analysis shown by Table~\ref{tab:sensitivity_meas_noise}. Using the same trajectory (see Fig.~\ref{fig:lissajous_traj_example_3d}, \ref{fig:lissajous_traj_example_2d}), we are showing one case, which fails because of the predominant measurement noise setting of \SI{3}{\meter\per\second} (false selection case indicated by the table), and one successful detection using \SI{1.5}{\meter\per\second} measurement noise.

\begin{figure}[!h]
    \vspace{-3mm}
    \centering
    \includegraphics[width=1\linewidth, trim={0 0mm 0 0mm},clip]{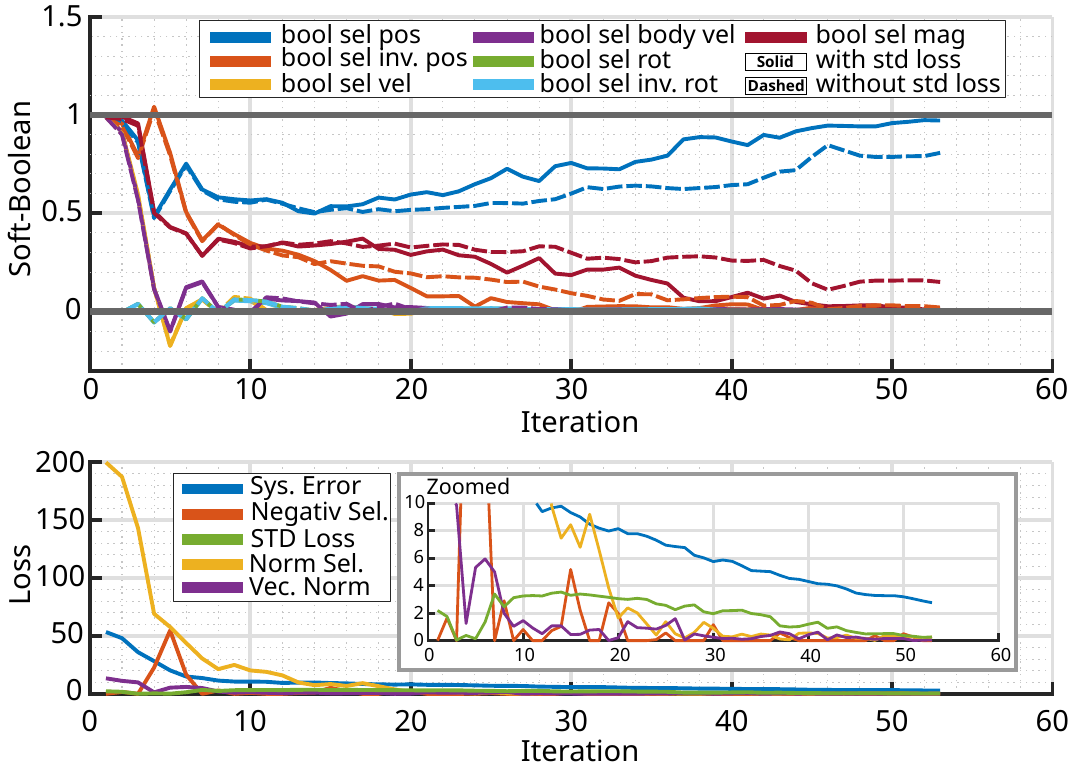}
    \caption{Experimental comparison and further proof of Theorem~\ref{lemma:std_reg} for the STD loss term, which enforces a single model choice.
    The data for this example is perturbed with noise to promote a case that tends towards an ambiguous decision.
    \textbf{Top:}~Convergence of soft Booleans with (solid) and without (dashed) the STD loss term.
    It can be seen that a fully accurate decision is not reached if the loss term is not enabled (Position Model: $0.8$, Magnetometer: $0.2$).
    This further confirms the exact scenario of possible over-fitting (Sec.~\ref{sec:method}) if dedicated loss terms are not defined.
    \textbf{Bottom:}~Convergence of the system error and all loss terms for the case where the STD loss is used.
    }
    \label{fig:experiment_convergence}
\end{figure}

\begin{figure}[!h]
    \centering
    \includegraphics[width=1.0\linewidth, trim={0 0mm 0 0mm},clip]{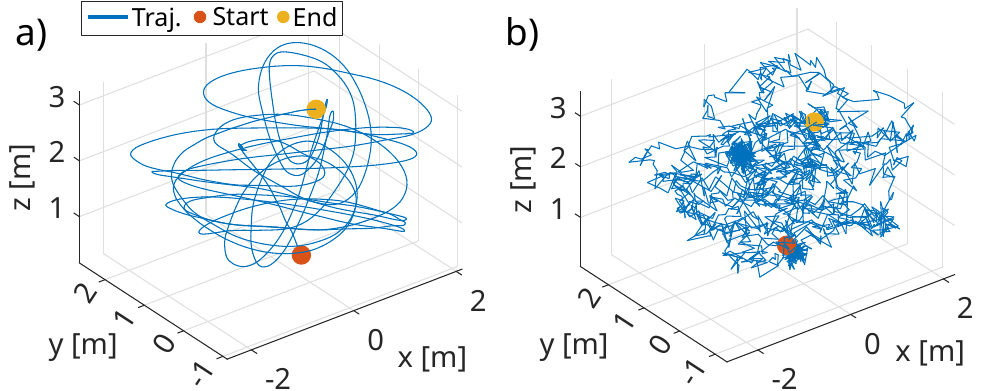}
    \caption{3D Lissajous trajectory for one of ten test datasets. Figure a) without noise on the system input (core states) and Figure b) with noise, as defined in the introduction of Sec.~\ref{sec:experiments}.}
    \label{fig:lissajous_traj_example_3d}
\end{figure}

\begin{figure}[!h]
    \centering
    \includegraphics[width=0.97\linewidth, trim={0 0mm 0 0mm},clip]{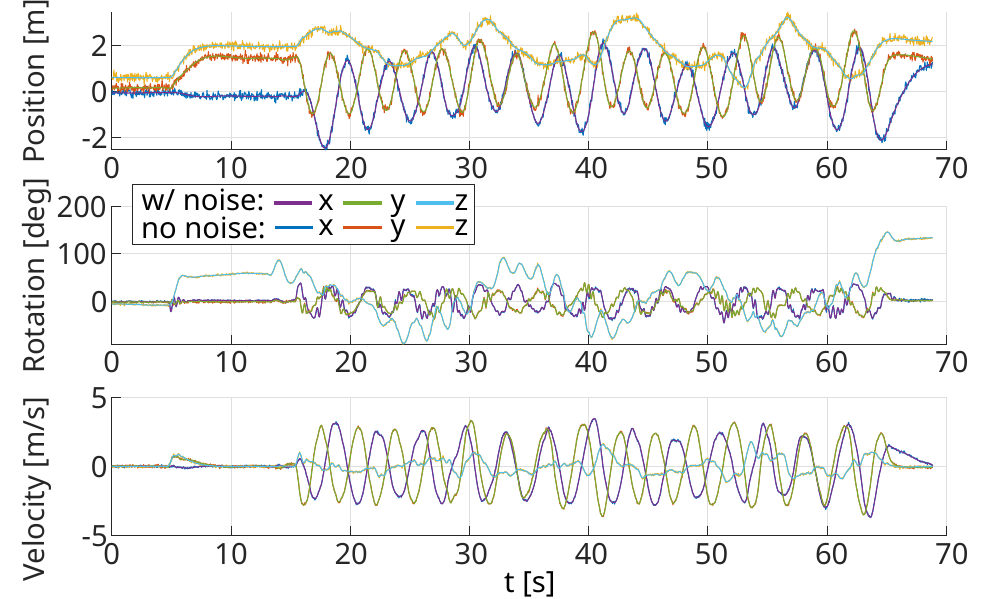}
    \vspace*{-1mm}
    \caption{2D position, rotation and velocity, corresponding to Figure~\ref{fig:lissajous_traj_example_3d}, as the base localization information for the vehicle.}
    \label{fig:lissajous_traj_example_2d}
\end{figure}

\begin{figure}[!h]
    \centering
    \vspace{-4mm}
    \includegraphics[width=0.97\linewidth, trim={0 0mm 0 0mm},clip]{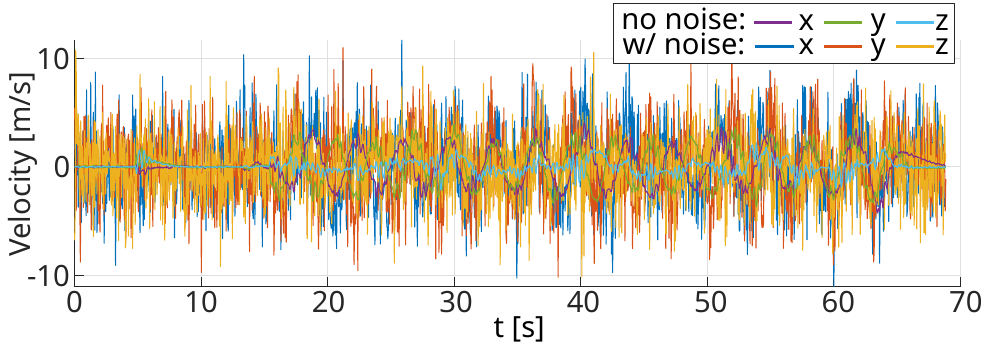}
    \vspace*{-1.5mm}
    \caption{Ground truth velocity measurement in relation to the noise-perturbed signal ($\sigma = \SI{3}{\meter\per\second}$), causing an incorrect model identification.}
    \label{fig:sim_vel_measurement}
\end{figure}

\begin{figure}[!h]
\vspace{-5mm}
    \centering
    \includegraphics[width=0.97\linewidth, trim={0 0mm 0 0mm},clip]{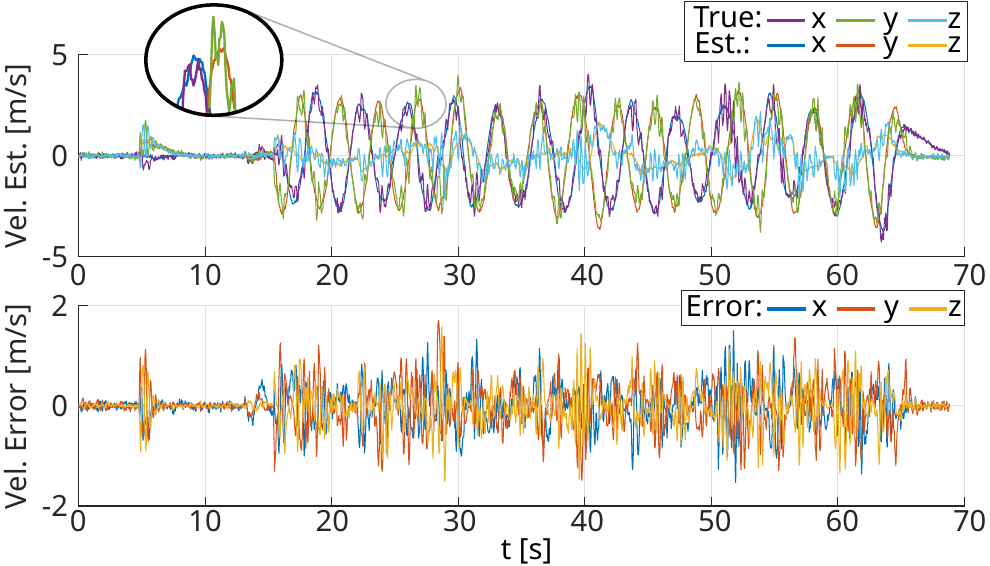}
    \vspace*{-1mm}
    \caption{Reprojection of the incorrectly identified position sensor model onto the ground truth velocity measurement. The estimated model loosely resembles the true signal but shows strong oversimplifications as highlighted by the zoomed section and the error plot.}
    \label{fig:sim_vel_ident_failed}
\end{figure}

\begin{figure}[!h]
    \vspace{-5mm}
    \centering
    \includegraphics[width=0.97\linewidth, trim={0 0mm 0 0mm},clip]{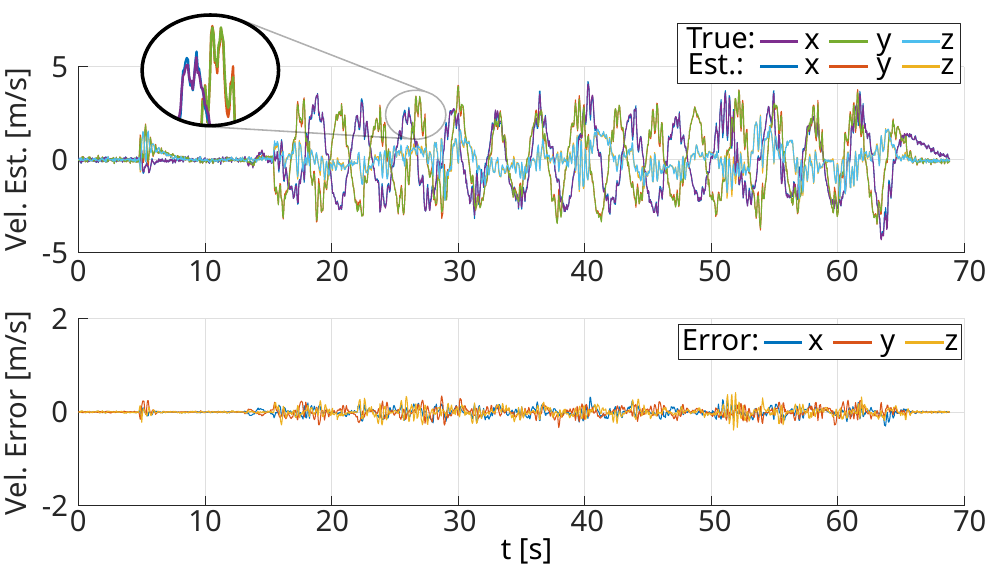}
    \vspace*{-1mm}
    \caption{Reprojection of the correctly identified and parameterized velocity sensor model. The estimated data matches the true signal closely, confirmed by the error plot with a significantly lower magnitude compared to Figure~\ref{fig:sim_vel_ident_failed}.}
    \label{fig:sim_vel_ident_success}
\end{figure}

\clearpage

Figure~\ref{fig:sim_vel_measurement} shows the velocity measurement for the failing case and relates the noise-perturbed measurement to the ground truth signal.
In this case, the sensor model was incorrectly identified as a position sensor, Figure~\ref{fig:sim_vel_ident_failed} shows that the reprojected measurement, based on the model and estimated parameterization resembles the true value to some extent. However, the zoomed element shows that the signal is strongly oversimplified. To the contrary, Figure~\ref{fig:sim_vel_ident_success} for the successful detection shows that the estimated signal matches the ground truth of the measurement closely. Comparing the error plots of both cases also shows a significant difference in magnitude.

\begin{table}[!h]
\setlength\tabcolsep{3pt}
    \small
    \ra{0.8}
    \caption{Results of the two example identification cases shown by Fig.~\ref{fig:sim_vel_ident_failed} and \ref{fig:sim_vel_ident_success}. Section~\ref{sec:healt_metric} will show that the introduced health metric can detect the false model selection reliably.
    }%
    \vspace{-3mm}
    \label{tab:validity_rmse}
    \begin{center}
        \begin{tabular}{ l c c c c c c c}
        & & & \multicolumn{2}{c}{\textbf{Soft Boolean}} & \multicolumn{3}{c}{\textbf{RMSE}} \\
        \cmidrule(lr){4-5} \cmidrule(lr){6-8}
        & \textbf{Noise} & \textbf{Model Sel.} & Selector  & Distance & x & y & z \\ 
         \midrule
         \textbf{Case 1} & \SI{3}{\meter} & Pos. & 0.548 & \textbf{0.082} & 3.32 & 3.33 & 3.39  \\ 
         \midrule
         \textbf{Case 2} & \SI{1.5}{\meter} & Vel. & 0.564 & \textbf{0.403} & 1.51 & 1.48 & 1.49 \\ 
         \bottomrule
        \end{tabular}
    \end{center}
    \vspace{-8mm}
\end{table}

\vspace{3mm}
Table~\ref{tab:validity_rmse} presents the results for both scenarios, soft Boolean values, the selected model, as well as the RMSE for the reprojection of the sensor model with the estimated parameterization. The next section will introduce a metric to identify false positives and show that such failure cases can be detected and rejected reliably.
We will show that the distance between the highest and second highest boolean is one of the main parameters for this identification and that the example case presented here, with a false detection and a Boolean distance of 0.082 would be rejected accordingly.

\vspace{-2mm}
\subsection{Health Metric}
\label{sec:healt_metric}

\begin{figure}[!b]
    \centering
    \vspace{-5mm}
    \includegraphics[width=1\linewidth, trim={0 0mm 0 0mm},clip]{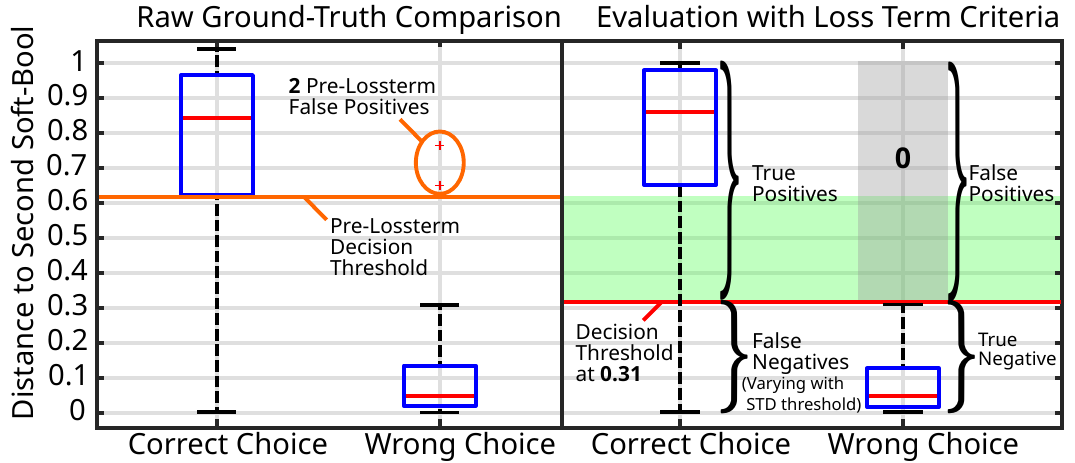}
    \caption{Box plot for soft Boolean distances $\Delta \mathsf{b}$ (difference between highest and second highest) for correct and wrong model selections based on ground truth.
    \textbf{Left:}~Raw samples, 1690 correct and 60 incorrect selections with two false positives for the wrong choice category within 1750 samples. Figure~\ref{fig:loss_scatter_matrix} is used to determine thresholds for loss parameters in order to remove such outliers.
    \textbf{Right:}~Samples with removed false positives based on remaining loss term limits. The green area shows a decision range with corresponding trade-offs in precision and recall (Fig.~\ref{fig:precision_recall_traditional}, \ref{fig:precision_recall}). The best recall ($91\%$) at $100\%$ precision is reached with the \textbf{red} Bool-Decisions threshold $\Delta \mathsf{b} < 0.31$, $\mathcal{L}_\mathsf{norm} < 0.2$, and $\mathcal{L}_\mathsf{std} < 4.1$.
    \textbf{Data:}~1750 samples, from sensitivity analysis (Tab.~\ref{tab:sensitivity_snr_individual}). Based on the health metric, 603 samples are rejected, with no remaining false positives. Please keep in mind that selected samples are intended to fail due to high noise perturbations on the system inputs. \looseness=-1
    }
    \label{fig:false_positives_box_plot}
\end{figure}

\begin{figure}[!b]
    \centering
    \includegraphics[width=1.0\linewidth, trim={0 0mm 0 0mm},clip]{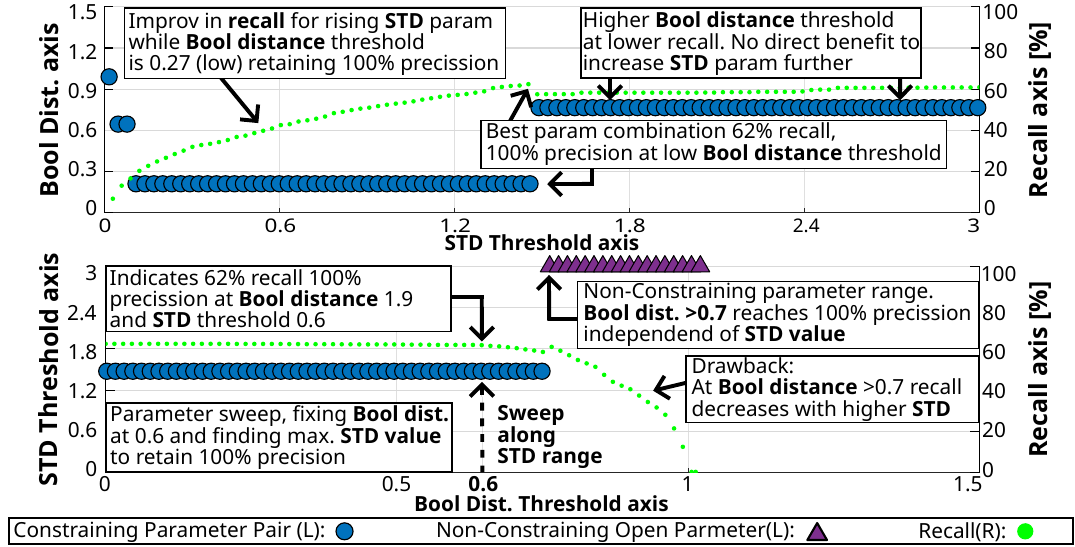}
    \caption{Detailing of Figure~\ref{fig:loss_scatter_matrix}, shows the parameter evaluation of Boolean distance threshold against STD threshold (bottom), and STD threshold against Boolean distance (top).
    The objective is to find the threshold with the lowest Boolean distance and highest STD value to reach the highest recall at \SI{100}{\percent} precision for the identification of false positives.
    Figure~\ref{fig:loss_scatter_matrix} shows this evaluation schematic for all parameters. Figure~\ref{fig:precision_recall_traditional} and~\ref{fig:precision_recall} show a traditional precision-recall graph, further refining the most meaningful parameters.}
    \label{fig:loss_scatter_matrix_explained}
\end{figure}

\begin{figure*}[t]
    \centering
    \includegraphics[width=1.0\linewidth, trim={0 0mm 0 0mm},clip]{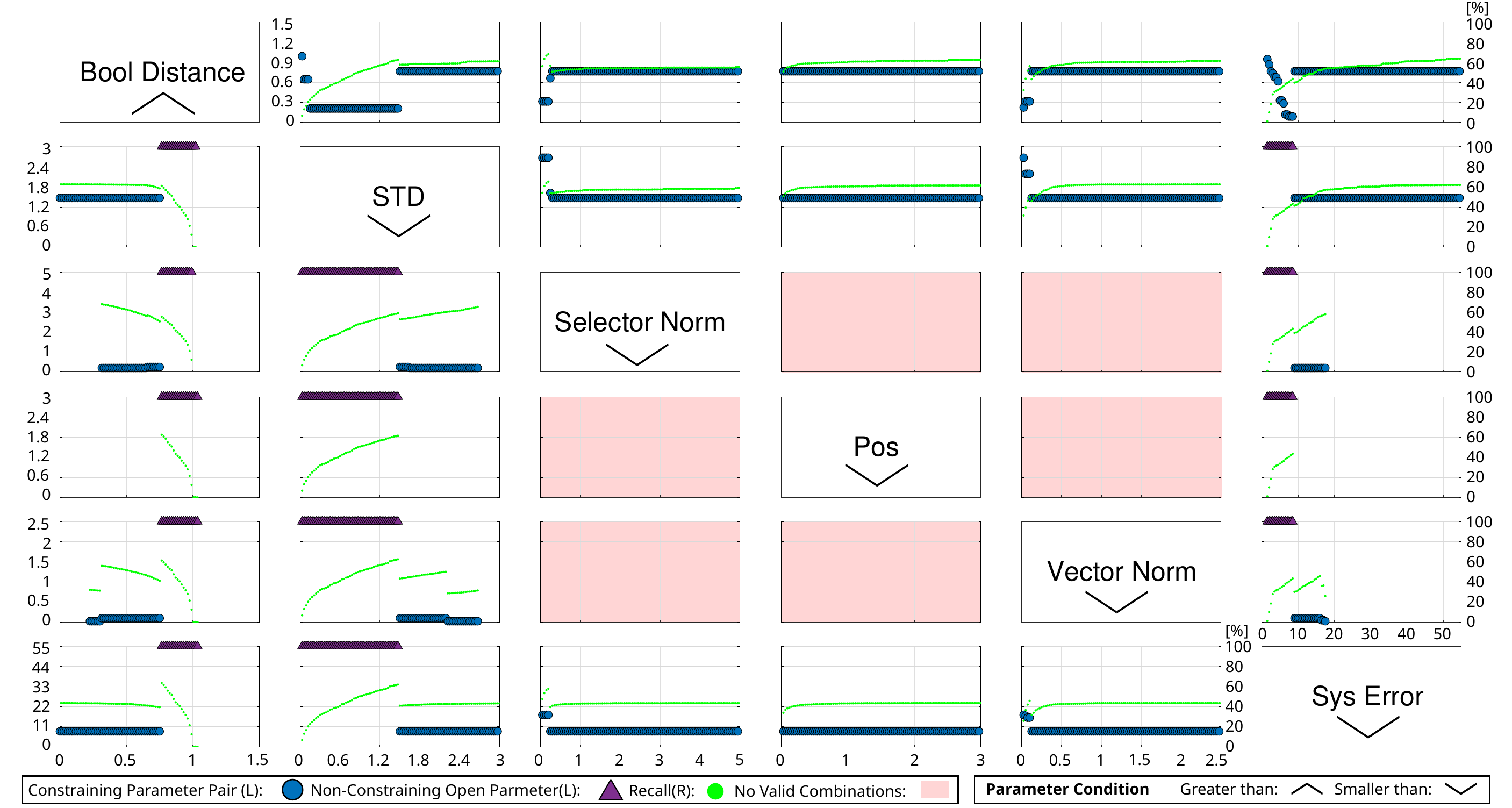}
    \caption{Matrix scatter plot with an overview of relevant final system parameters given 1750 samples from the previous sensitivity analysis.
    The graph is used to determine a precursor for rejecting outliers and improving false positives based on parameter correlation and recall.
    The diagonal labels determine the axis labels. 
    Parameters such as loss terms are rejected if they are smaller than ($\vee$) the threshold, and parameters such as Boolean distances are rejected if they are greater than ($\wedge$) the threshold.
    The left y-axis \textbf{(L)} relates to blue dots \protect\mpsympair, representing a constraining parameter pair, and purple triangles \protect\mpsymopen indicating that a parameter pair is not constraining. The right y-axis \textbf{(R)} relates to green dots \protect\mpsymrecall, representing the recall in $[\%]$ for the individual parameter pair.
    Red plots indicate no results for any parameter pair.}
    \label{fig:loss_scatter_matrix}
    \vspace{-5mm}
\end{figure*}

\begin{figure}[!tb]
    \centering
    \includegraphics[width=1\linewidth, trim={0 0mm 0 0mm},clip]{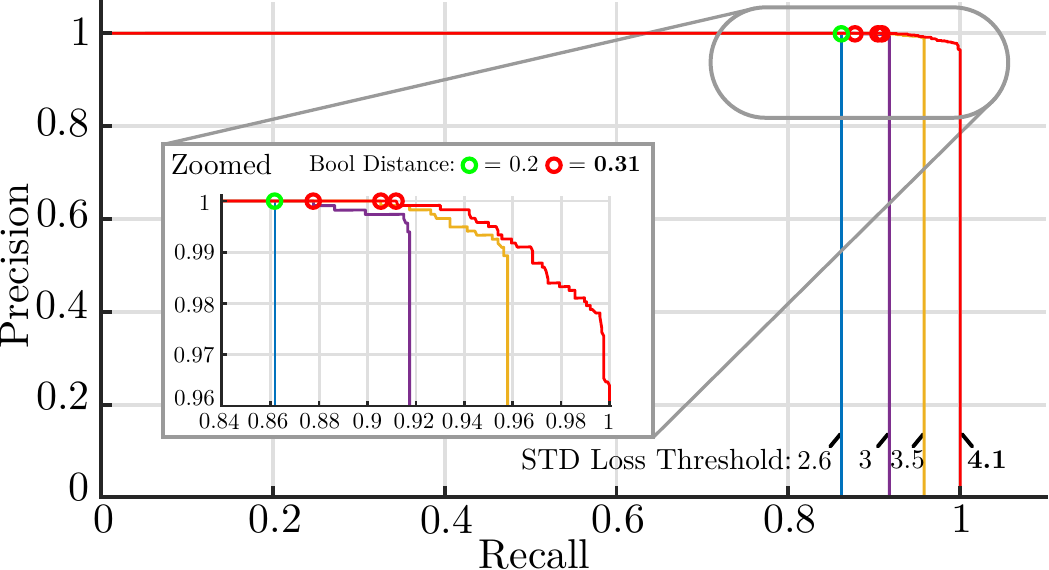}
    \caption{Precision/Recall, for dedicated $\mathcal{L}_\mathsf{std}$ over the soft Boolean distance threshold $0\text{-}1$ and markers at dedicated distances.
    $\mathcal{L}_\mathsf{std}<2.6$ rejects all wrong decisions independently of $\Delta \mathsf{b}$ at a lower recall $86\%$.
    The best result for $100\%$ precision is given at $\mathcal{L}_\mathsf{std}<4.1$ and a distance of $\Delta \mathsf{b} > 0.31$ with $91\%$ recall.}
    \label{fig:precision_recall_traditional}
    \vspace{-6mm}
\end{figure}

\begin{figure}[!tb]
    \centering
    \includegraphics[width=1\linewidth, trim={0 0mm 0 0mm},clip]{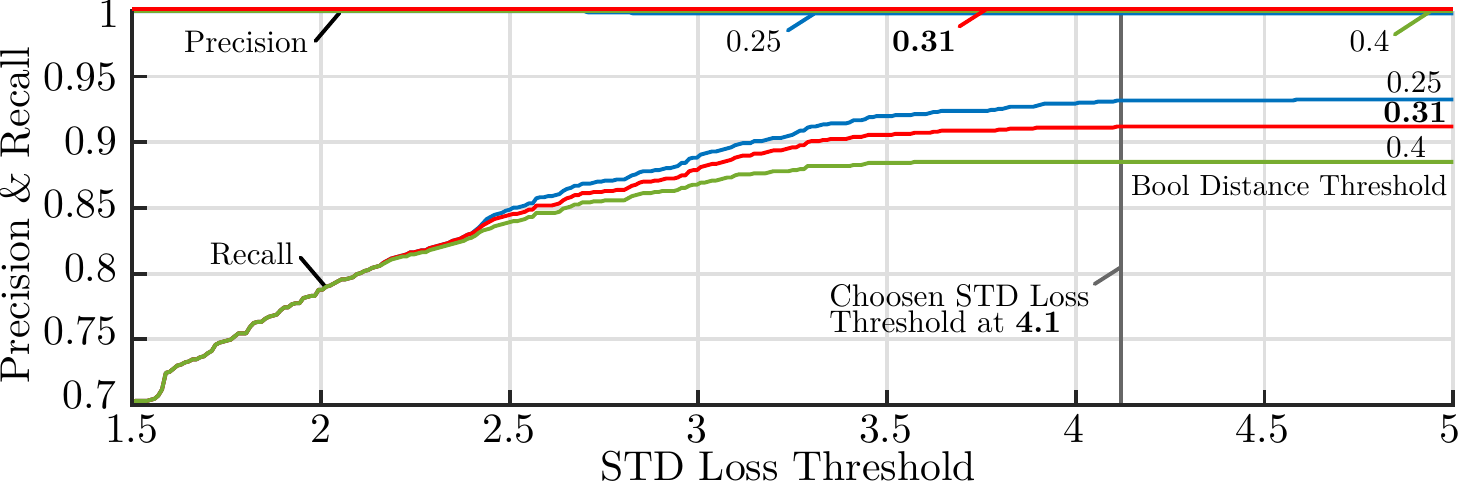}
    \caption{Precision/Recall, for dedicated soft Boolean distances ranging over $\mathcal{L}_\mathsf{std}$. The red precision graph for a distance of $\Delta \mathsf{b} = 0.31$ remains at $100\%$ precision below $\mathcal{L}_\mathsf{std}<4.1$ as the best threshold.}
    \label{fig:precision_recall}
    \vspace{-6mm}
\end{figure}

Since the problem definition for the presented method is not convex, a metric for detecting false decisions is vital.
The choice for a sensor model is made based on the soft Boolean with the highest magnitude. We define an additional health metric that uses the distance between the highest and second-highest soft Boolean as well as the final values of the loss functions, which define the soft Boolean properties.

The distance between the two highest soft Booleans correlates with the certainty of a correct selection. The box plot in Figure~\ref{fig:false_positives_box_plot} illustrates this relation. The left side shows the Boolean distances for correct and incorrect choices based on ground-truth selectors. The distances for the categorical ``wrong choice'' show a few outliers, but the first quantiles for both distances do not overlap.
Please keep in mind that the data used for this box plot is the same as for the results of the sensitivity analysis with perturbed SNR in Table~\ref{tab:sensitivity_snr_individual}. This data was designed to intentionally produce false choices due to, in part, close to unrealistic high noise perturbations on the system inputs.

Given the box plot, a naive choice for a decision threshold would be at $\Delta \mathsf{b} > 0.62$ labeled as the pre-loss term decision threshold in Figure~\ref{fig:false_positives_box_plot}. However, this includes two false positives and rejects an unnecessarily high number of true positives.
Thus, in addition to outlier rejection, we also want to define a metric that promotes a high recall (low false negatives) and $100\%$~precision (no false positives).
In the following, we introduce a metric to pre-condition the base of decision and filter decisions based on remaining loss terms, residual of the model, and the Boolean distance for each selection.
This health metric will be deployed by rejecting results for which any individual criteria does not hold and substantially improve the precision ($\frac{\textrm{true positives}}{\textrm{true positives} + \textrm{false positives}}$) and recall ($\frac{\textrm{true positives}}{\textrm{false negatives} + \textrm{true positives}}$) for the decision making process.

The parameter space for the identification of meaningful thresholds for the outlier detection in six dimensions is to high for numerical evaluations. Because of this, we are using a graphical representation in Figure~\ref{fig:loss_scatter_matrix} that uses a scatter plot matrix and the paradigm of precision and recall at its core. Here we can use multiple axes in two dimensions and identify opposing effects of e.g. remaining loss terms.
We can observe, which points of the value ranges do not retain \SI{100}{\percent} precision, and can observe this behavior over all remaining loss terms for ($\Delta \mathsf{b}$, $ \mathcal{L}_\mathsf{std}$, $\mathcal{L}_\mathsf{norm}$, $\mathcal{L}_\mathsf{pos}$, $\mathcal{L}_\mathsf{vec}$, and $f_\mathsf{sys}$).

The scatter plot matrix in Figure~\ref{fig:loss_scatter_matrix} shows the same raw data as the left side of the box plot (Fig.~\ref{fig:false_positives_box_plot}). Figure~\ref{fig:loss_scatter_matrix_explained} is intended to detail individual aspects based on the $\Delta \mathsf{b}$, $ \mathcal{L}_\mathsf{std}$ loss-terms.

The diagonal blocks of this matrix represent the pivot point of a specific parameter.
Scatter plots along a row of a respective diagonal entry present the specific parameter on their left y-axis, and plots along the column of this diagonal entry present the same parameter on the x-axis.
Each individual row and column have the same axis scaling.
Parameters such as the Boolean distance are included if they are greater than a threshold, indicated by $\wedge$, and parameters such as the remaining loss terms are included only if they are below a threshold, indicated by $\vee$.

Markers \mpsympair within the plots represent the combination of both parameters, which result in the highest recall at $100\%$~precision. The results are determined by fixing parameter values on the x-axis and increasing the parameter value corresponding to the y-axis to the last point where $100\%$~precision is given.
Because the parameter for the y-axis does not always pose a constraint in combination with the parameter of the x-axis, triangle markers \mpsymopen are used to indicate the case where no upper limit exists.
Fields marked in red indicate that no combination of the respective two parameters allow $100\%$~precision.
Outliers can be visually identified by cut-offs, causing longer straight lines, at higher value ranges either on the upper or lower triangular of the full plot matrix.

The purpose of this plot matrix is to identify parameter combinations of interest for further evaluation.
Because the marker for parameter pairs do not provide any information about the corresponding recall value, data points~\mpsymrecall for individual recall values ranging from~$0\text{-}\SI{100}{\percent}$ according to the right y-axis are added.

We are using loss terms, residuals, and the Boolean distance for this metric. Loss terms close to zero represent the most accurate results, and high loss terms indicate that particular properties and constraints could not be retained.
The opposite holds for the Boolean distance. The bigger the distance, the more certain the decision is.
Generally, we require thresholds that are as constraining as necessary but as inclusive as possible to retain a high recall.

The matrix plot shows that there is no meaningful combination between the loss-terms $\mathcal{L}_\mathsf{norm}$, $\mathcal{L}_\mathsf{pos}$, and $\mathcal{L}_\mathsf{vec}$ because there is no parameter combination to produce a valid result at $100\%$ precision.
Outliers can be visually inferred for \mbox{$\Delta \mathsf{b} = 0.65$}, \mbox{$\mathcal{L}_\mathsf{std}=1.49$}, \mbox{$\mathcal{L}_\mathsf{norm}=0.2$}, and \mbox{$f_\mathsf{sys} = 8.57$}.
The matrix plot indicates that the most meaningful parameters for outlier rejection and maximization of the recall at $100\%$ precision are $\mathcal{L}_\mathsf{std}$ and $\mathcal{L}_\mathsf{norm}$ because it shows the highest recall at $65\%$.

We can also see that $\mathcal{L}_\mathsf{std}$ together with $\Delta \mathsf{b}$ indicates a more dynamic parameter range that can allow for further improvement.
Thus, we are applying the first criteria to the individual results ($\mathcal{L}_\mathsf{norm}<0.2$) and generate precision/recall graphs for $\Delta \mathsf{b}$, depending on $\mathcal{L}_\mathsf{std}$ to further detail their effect on the recall after the precursor was applied.

Figure~\ref{fig:precision_recall_traditional} shows the precision/recall graph for dedicated $\mathcal{L}_\mathsf{std}$ thresholds while varying the Boolean distance threshold.
It can be seen that a strong limitation of $\mathcal{L}_\mathsf{std} < 2.6$ eliminates all incorrect choices at $100\%$ precision, but at the cost of a low recall due to unnecessarily removed true samples.
The best recall ($91\%$) while maintaining $100\%$~precision is reached for \mbox{$\mathcal{L}_\mathsf{std} > 4.1$} and \mbox{$\Delta \mathsf{b} = 0.31$}.

Figure~\ref{fig:precision_recall} shows the same data for individual Boolean distances but varying $\mathcal{L}_\mathsf{std}$. A recall plateau is reached for higher STD loss thresholds.
Please note that $\mathcal{L}_\mathsf{std}$ does not decrease the Boolean threshold further. However, it reduces the number of samples in the false negative and true negative categories, thus allowing for an increased recall and a redundant multi-parameter method, providing a robust health metric.
The right box plot of Figure~\ref{fig:false_positives_box_plot} shows the statistic for the Boolean distances with the full pre-filtering of samples applied $\mathcal{L}_\mathsf{norm} > 0.2, \mathcal{L}_\mathsf{std} > 4.1, \Delta \mathsf{b} = 0.31$.
The green area marks the improvement of the decision range for correct choices, providing zero false positives and corresponding recalls according to Figures~\ref{fig:precision_recall_traditional} and \ref{fig:precision_recall}.

The presented health metric allows for the detection of all false choices with zero false positives.
Should a selection be identified as incorrect, the procedure for stage~1 can be re-run with modified initial values to reach a correct choice within the optimization landscape.
Therefore, the presented method can be used with raw data in contrast to \cite{Brunton2016} and \cite{Zheng2019}, which requires a pre-processing of data and possible loss of information to improve noise resilience.

\subsection{Reference Frame Determination and Calibration States}
\label{sec:ref_frame_determination}
We will demonstrate the benefit of the reference frame determination based on the position sensor module.
Calibration states are defined by the IMU to sensor translation \mbox{$\mathbf{p}_\mathsf{is} = \left[0.3,0.5,1.0\right]$}
and, in case one is given, the reference frame \mbox{$\mathbf{p}_\mathsf{rw} = \left[ 10,0,0 \right]$} and 
\mbox{$\boldsymbol{\omega}_\mathsf{rw}=\left[0,0.8727,0 \right]$} corresponding to \mbox{$\left[0,50,0\right]$} degree in Euler angles.

The evaluation is done based on four scenarios, with 10 different trajectories and numerically illustrated by Table~\ref{tab:ref_frame_test}.
Two scenarios in which no reference frame is required, once without (w/o) additional constraint and once with $\mathcal{L}_\mathsf{ref}$ applied \mbox{(w/$\mathcal{L}_\mathsf{ref}$)}. The two other scenarios, in contrast do require a reference frame and are also evaluated once without reference frame constraint and once with $\mathcal{L}_\mathsf{ref}$.

Table~\ref{tab:ref_frame_test} shows the improvement in the cases where the additional loss term $\mathcal{L}_\mathsf{ref}$ was applied. Refining the reference frame also positively affects the accuracy of the calibration $\mathbf{p}_\mathsf{is}$.
In case a reference frame exists, the accuracy of this frame is improved. In case no reference frame is required, the basis for this decision becomes more clear due to a marginal remaining norm of $3\!\cdot\!10^{-5}$.

\begin{table}[!tb]
\setlength\tabcolsep{0.75pt}
\footnotesize 
    \caption{Comparison of final calibration and reference states for the position sensor module. The table shows absolute values (e.g. $\mathbf{p}_\mathsf{ip}$) and individual errors w.r.t. the true calibration (e.g. $\mathbf{\tilde{p}}_\mathsf{ip}$), and the corresponding norm. The evaluation presents the average result based on 10 different trajectories.
    The comparison shows four cases for cross-evaluation where a reference frame was required or not, and the loss term $\mathcal{L}_\mathsf{ref}$ was applied or not.
    In the case of the reference frames, applying the additional loss provides a more reliable basis for the decision if a reference frame needs to be modeled. $\boldsymbol{\omega}$ represents the rotation in the tangent space and radians and $\lVert\boldsymbol{\omega}\rVert$ is the angle error. 
    \looseness=-1}%
    \label{tab:ref_frame_test}
    \begin{center}
    \vspace{-2mm}
\begin{tabular}{|c|c|c|c|c|c|c|c|c|c|c|c|c|}
\hline 
 & \multicolumn{2}{c|}{w/o} & \multicolumn{2}{c|}{w/$\mathcal{L}_\mathsf{ref}$} & \multicolumn{2}{c|}{w/o} & \multicolumn{2}{c|}{w/$\mathcal{L}_\mathsf{ref}$} & \multicolumn{2}{c|}{w/o} & \multicolumn{2}{c|}{w/$\mathcal{L}_\mathsf{ref}$}\tabularnewline
\hline 
 & $\mathbf{p}_\mathsf{ip}$ & $\mathbf{\tilde{p}}_\mathsf{ip}$ & $\mathbf{p}_\mathsf{ip}$ & $\mathbf{\tilde{p}}_\mathsf{ip}$ & $\mathbf{p}_\mathsf{rw}$ & $\mathbf{\tilde{p}}_\mathsf{rw}$ & $\mathbf{p}_\mathsf{rw}$ & $\mathbf{\tilde{p}}_\mathsf{rw}$ & $\boldsymbol{\omega}_\mathsf{rw}$ & $\boldsymbol{\tilde{\omega}}_\mathsf{rw}$ & $\boldsymbol{\omega}_\mathsf{rw}$ & $\boldsymbol{\tilde{\omega}}_\mathsf{rw}$\tabularnewline
\hline
\multicolumn{13}{|c|}{No reference frame required} \\
\hline
x                 & 0.298 & 0.027 & 0.300 & 0.013 & 0.017 & 0.057 & 9e-6 & 3e-5 & 0.013 & 0.025 & 9e-4 & 0.002 \tabularnewline
\hline 
y                 & 0.497 & 0.022 & 0.494 & 0.014 & 0.033 & 0.090 & 8e-7 & 2e-6 & 0.007 & 0.028 & 7e-5 & 0.004 \tabularnewline
\hline 
z                 & 1.009 & 0.034 & 1.000 & 0.022 & 0.028 & 0.071 & 9e-8 & 7e-8 & 0.003 & 0.025 & 0.002 & 0.005 \tabularnewline
\hline 
$\lVert . \rVert$ & 1.164 & 0.048 & 1.155 & \textbf{0.029} & 0.047 & 0.128 & 1e-7 & \textbf{3e-5} & 0.015 & 0.045 & 0.002 & \textbf{0.008} \tabularnewline
\hline 
\hline 
\multicolumn{13}{|c|}{With reference frame required} \\
\hline
x                 & 0.294 & 0.026 & 0.315 & 0.025 & 9.998 & 0.066 & 9.915 & 0.109 & 0.014 & 0.038 & 0.003 & 0.008
\tabularnewline
\hline 
y                 & 0.503 & 0.026 & 0.502 & 0.023 & 0.002 & 0.098 & 0.002 & 0.005 & 0.880 & 0.038 & 0.879 & 0.023
\tabularnewline
\hline 
z                 & 0.981 & 0.064 & 1.032 & 0.062 & 0.032 & 0.121 & 0.003 & 0.012 & 0.003 & 0.037 & 0.002 & 0.013
\tabularnewline
\hline 
$\lVert . \rVert$ & 1.141 & 0.074 & 1.190 & \textbf{0.070} & 9.998 & 0.169 & 9.915 & \textbf{0.110} & 0.880 & 0.065 & 0.879 & \textbf{0.029} \tabularnewline
\hline 
\end{tabular}
    \end{center}
\vspace{-5mm}
\end{table}
The error of the final calibration states is evaluated based on the same 10 trajectories as for the reference frame determination.
Table~\ref{tab:calib_state_test} shows the averaged errors for the individual calibrations.
The absolute calibration errors show sufficient accuracy for the given noise parameters and thus, the initial guess of the calibration states are beneficial for the initialization of e.g., an EKF.
\begin{table}[!tb]
\setlength\tabcolsep{5pt}
\footnotesize 
    \caption{Average extrinsic calibration errors for all sensor modules based on 10 trajectories. $\boldsymbol{\omega}$ represents the rotation in the tangent space in radians and $\lVert\boldsymbol{\omega}\rVert$ is the angle error.}%
    \label{tab:calib_state_test}
    \begin{center}
    \vspace{-2mm}
\begin{tabular}{|c|c|c|c|c|c|c|c|c|}
\hline 
 & \multicolumn{1}{c|}{Pos.} & \multicolumn{2}{c|}{Inv. Pos.} & \multicolumn{1}{c|}{Vel.} & \multicolumn{2}{c|}{Body Vel.} & \multicolumn{2}{c|}{Mag.} \tabularnewline
\hline 
 & $\mathbf{\tilde{p}}_\mathsf{ip}$ &
 $\mathbf{\tilde{p}}_\mathsf{ip}$ & $\boldsymbol{\tilde{\omega}}_\mathsf{ip}$ & $\mathbf{\tilde{p}}_\mathsf{iv}$ & 
 $\mathbf{\tilde{p}}_\mathsf{iv}$ & $\boldsymbol{\tilde{\omega}}_\mathsf{iv}$
 &$\mathbf{\tilde{m}}_\mathsf{w}$ & $\boldsymbol{\tilde{\omega}}_\mathsf{im}$ \tabularnewline
\hline
x                 & 0.026 & 0.076 & 0.016 & 0.002 & 0.008 & 0.004 & 0.014 & 0.023 \tabularnewline
\hline 
y                 & 0.021 & 0.062 & 0.024 & 0.004 & 0.006 & 0.006 & 0.011 & 0.008 \tabularnewline
\hline 
z                 & 0.033 & 0.141 & 0.012 & 0.003 & 0.012 & 0.009 & 0.009 & 0.031 \tabularnewline
\hline 
$\lVert . \rVert$ & 0.048 & 0.172 & 0.031 & 0.005 & 0.013 & 0.012 & 0.019 & 0.039 \tabularnewline
\hline 

\end{tabular}
\vspace{-8mm}
    \end{center}
\end{table}

\subsection{Real-World Application (UAV)}
\label{sec:experiment_real_world_uav}

The previous sections used simulated noise-affected data with dynamic trajectories to demonstrate and evaluate the presented method on sensor model identification.
The following experiment is intended to show the application of this approach to real-world data. For this purpose, we use a larger-scale outdoor-to-indoor transition trajectory from the INSANE datasets \cite{Brommer2024_insane}.
The dataset used a quadcopter that was designed for non-agile flight maneuvers and therefore, the flight segment of the trajectory has only little dynamic movement and no noteworthy orientation excitation.
However, the trajectory does have an initial segment where the quadcopter was excited by hand in 6-DoF, and then placed on the ground for a consecutive flight.
\looseness=-1

\begin{figure}[bp]
    \centering
    \vspace{-5mm}
    \includegraphics[width=0.95\linewidth, trim={0 0mm 0 0mm},clip]{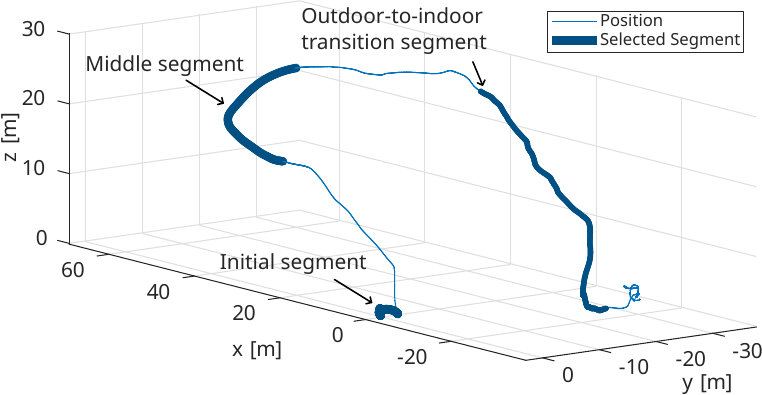}
    \caption{3D position of the full trajectory, estimated with the MaRS framework. The bold highlighted sections are the initial and middle segments used for the sensor model identification experiment.}
    \label{fig:smi_realworld_3d_pos}
\end{figure}

The objective is to detect a global velocity sensor, representing a vector element, and a magnetometer as a rotational component.
The magnetometer (@\SI{80}{\hertz}) is mounted close to the center of the vehicle. The ground truth calibration, given by the dataset, was performed with dedicated rotational excitation. The velocity measurements are provided by a GNSS sensor (@\SI{8}{\hertz}), which is mounted off-center on a larger beam at a distance of \SI{0.56}{\meter} from the IMU.
The localization information of the vehicle is generated, using the MaRS framework with two GNSS sensors as an input. The sensor data for the gray-box identification remained unprocessed and is directly passed to the identification approach.

\begin{figure}[tb]
    \centering
    \includegraphics[width=1\linewidth, trim={0 0mm 0 0mm},clip]{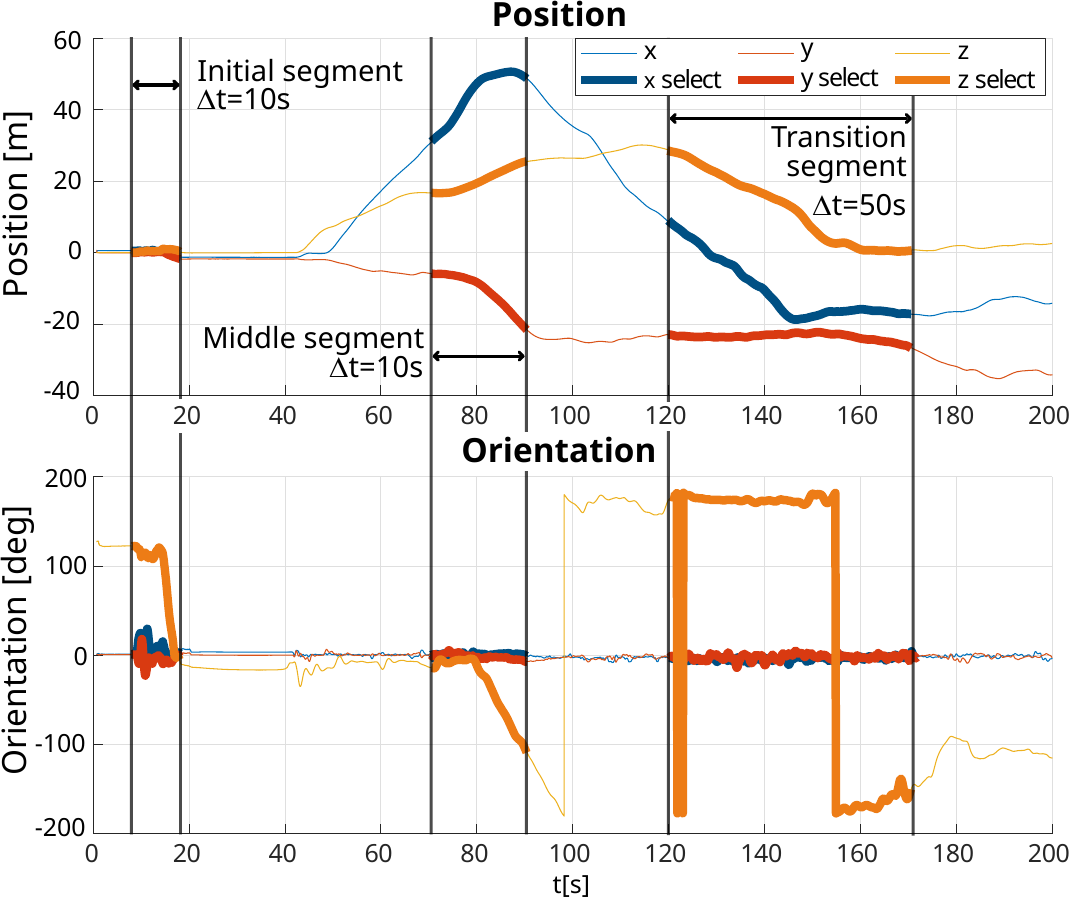}
    \caption{2D position and rotation estimates of the trajectory shown by Figure~\ref{fig:smi_realworld_3d_pos}.}
    \label{fig:smi_realworld_pose}
    \vspace{-4mm}
\end{figure}

We chose to perform the experiment for the magnetometer and velocity sensor identification for three segments. One for the initial phase, where the vehicle was excited by hand, representing a realistic scenario where an unknown sensor model needs to be identified and initialized on the ground, a second segment, where a sensor model needs to be identified mid-flight, and a more difficult third outdoor-to-indoor transition segment for the velocity model, where the velocity signal degrades towards the end of the trajectory.
The graphs in Figure~\ref{fig:smi_realworld_3d_pos} and \ref{fig:smi_realworld_pose} show the estimated position and orientation for the full trajectory where the selected segments are highlighted.
The initial segment is \SI{10}{\second} long (\SI{8}{\second}-\SI{18}{\second}), with 870 samples for the magnetometer and 80 samples for the velocity sensor. The middle segment is \SI{20}{\second} long (\SI{70}{\second}-\SI{90}{\second}), with 1756 magnetometer and 160 velocity samples, and the transition segment is \SI{50}{\second} long (\SI{120}{\second}-\SI{170}{\second}) with 400 velocity samples.
The middle segment is chosen to be longer and for a period where a rotation in yaw occurred because we need to increase the information given by the unknown signal, given the moderate flight pattern. As mentioned before, for real-world scenarios, the addition of sensors mid-flight could benefit from intentionally performing short maneuvers that increase the quality of observability for improved detection~\cite{Preiss2018}.

\begin{figure}[t]
    \centering
    \includegraphics[width=1\linewidth, trim={0 0mm 0 0mm},clip]{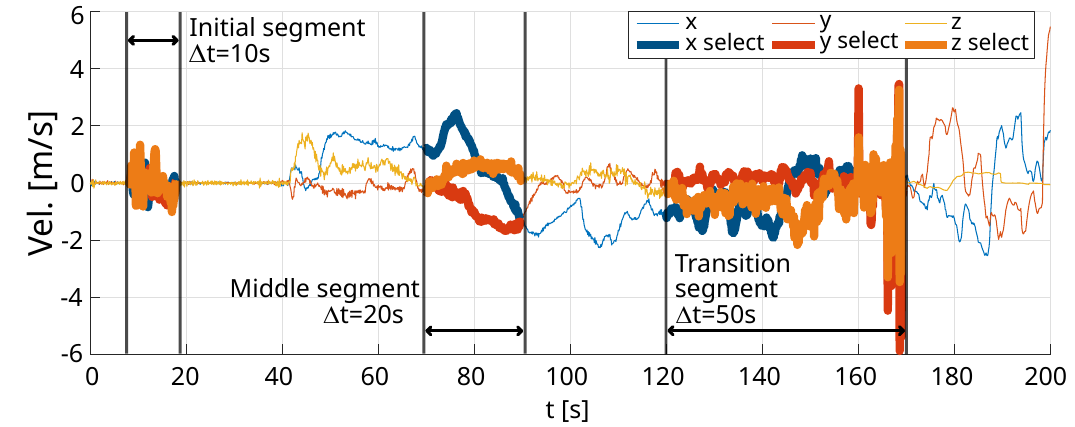}
    \vspace{-5mm}
    \caption{Unprocessed GNSS velocity measurements with the highlighted section, which are used for the sensor model identification experiment.}
    \label{fig:smi_realworld_vel}
    \vspace{-5mm}
\end{figure}

\begin{figure}[t]
    \centering
  \includegraphics[width=1\linewidth, trim={0 0mm 0 0mm},clip]{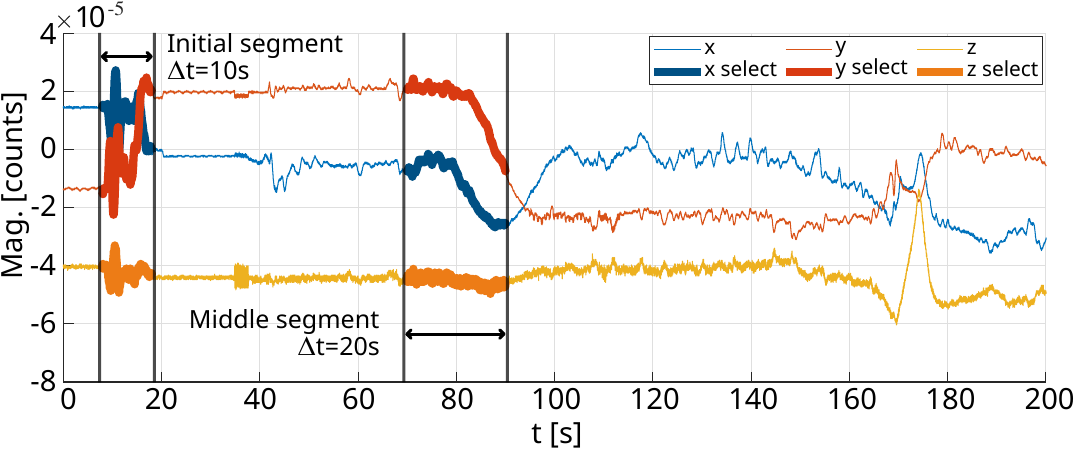}
    \caption{Unprocessed magnetometer sensor measurements in counts. The highlighted section is used for the sensor model identification experiment.}
    \label{fig:smi_realworld_mag}
    \vspace{-2mm}
\end{figure}

Figure~\ref{fig:smi_realworld_vel} and~\ref{fig:smi_realworld_mag} show the raw magnetometer and velocity measurements with the highlighted test segments.
The estimated core state and the unknown sensor signal samples are synchronized to the sensor timestamps and used as input to the sensor model identification approach.

The sensor model identification was successful in all four cases. The evaluation of the experiments is done based on a number of parameters: The final soft Booleans for the selection, the error of the calibration states, and the RMSE for the sensor measurement compared to the resembled measurement based on the parameterized sensor model.

\begin{table}[tb]
    \centering
    \begin{tabular}{ r c c c c c}
         & \multicolumn{2}{c}{\textbf{Mag.}} & \multicolumn{3}{c}{\textbf{Vel.}}\\
        \cmidrule(lr){2-3} \cmidrule(lr){4-6}
        \textbf{Soft Boolean}& Init. & Mid.  & Init. & Mid. & Transit.\\
        \midrule
        Selector        & 0.971 & 1.007 & 0.852 & 0.998 & 0.812 \\
        Dist. to second & 0.945 & 0.982 & 0.772 & 0.998 & 0.626 \\
        \bottomrule
    \end{tabular}
    \caption{Value of the soft Boolean selectors and the distance to the second highest boolean for both sensors and trajectory segments. The boolean distance is a measure for the accuracy of the selection as described in Section~\ref{sec:healt_metric}.}
    \label{tab:smi_boolean_realworld}
    \vspace{-0.4cm}
\end{table}

\begin{table}[!b]
    \centering
    \vspace{-3mm}
    \setlength\tabcolsep{3pt}
    \begin{tabular}{cccccc}
       & & \multicolumn{2}{c}{Init.} & \multicolumn{2}{c}{Mid.}\\
       \cmidrule(lr){3-4} \cmidrule(lr){5-6}
       $\rvar{i}{m}$ \textbf{[deg]} & \textbf{True} & \textbf{Est.} & \textbf{Error} & \textbf{Est.} & \textbf{Error}\\
       \midrule
       x & ~0.260 & 0 &  ~0.260 & 0 & ~0.260  \\
       y & ~2.264 & 0 & ~2.264 & 0 & ~2.264\\
       z & -1.882 & 0 & -1.883 &  0 & -1.882\\
       \midrule
       $\vvar{m}[w]$ \textbf{[deg]} & & & & & \\
       Incl. & ~85.178 & ~81.747 & ~-3.431 & ~92.290 & ~~7.112\\
       Decl. & -46.330 & -62.471 & -16.141 & -62.016 & -15.686\\
       \bottomrule
    \end{tabular}
    \caption{Comparison of the true magnetic field vector expressed in the world frame $\vvar{m}[w]$ and the rigid-body rotation $\rvar{i}{m}$ to the estimated results, using the parameterized sensor model, identified by the sensor model identification approach.
    }
    \label{tab:smi_realworld_mag}
\end{table}

The final boolean selector and the distance to the second highest boolean, shown by Table~\ref{tab:smi_boolean_realworld}, are unambiguous and true-positives according to the health metric (see Sec.~\ref{sec:healt_metric}).
Tables~\ref{tab:smi_realworld_mag} and~\ref{tab:smi_realworld_vel} show the estimated calibration state of the two sensors after stage two of the identification process (see Fig.~\ref{fig:method_outline}).
The calibration for the magnetometer shows minor errors for the estimation of the rotation for the rigid-body calibration $\rvar{i}{m}$, which has an angle error of $\SI{2.953}{\degree}$ for both, the initial and middle part of the trajectory. However, the declination of the estimated magnetic field expressed in the world frame $\vvar{m}[w]$ has a noticeable error of $\SI{15}{\degree}$, which is due to only minimal excitation of the roll and pitch angle for both data segments.

\begin{table}[t]
    \centering
    \setlength\tabcolsep{3pt}
    \begin{tabular}{cccccccc}
       & & \multicolumn{2}{c}{Init.} & \multicolumn{2}{c}{Mid.} & \multicolumn{2}{c}{Transit.}\\
       \cmidrule(lr){3-4} \cmidrule(lr){5-6} \cmidrule(lr){7-8}
       $\vvar{p}[i][v]$ \textbf{[m]} & \textbf{True} & \textbf{Est.} & \textbf{Error} & \textbf{Est.} & \textbf{Error} & \textbf{Est.} & \textbf{Error}\\
       \midrule
       x & 0.35 & 0.356 & -0.006 & 0.449 & -0.099 & ~0.866 & 0.516 \\
       y & 0.41 & 0.420 & -0.010 & 0.697 & -0.287 & -1.775 & 2.185 \\
       z & 0    & 0.212 & -0.212 & 0.288 & -0.288 & ~1.065 & 1.065 \\
       \bottomrule
    \end{tabular}
    \caption{
    Estimated and measured calibration of the rigid body translation between the IMU and the velocity sensor $\vvar{p}[i][v]$. The error of the z-component is slightly higher due to the zero z-component of the true calibration and an insufficient orientation excitation around the directional vector from the IMU to the sensor frame.}
    \label{tab:smi_realworld_vel}
\end{table}

\begin{table}[tbp]
    \centering
    \begin{tabular}{ c c c c c c}
        \cmidrule(lr){2-3} \cmidrule(lr){4-6}
        \textbf{RMSE} & Init. & Mid.  & Init. & Mid. & Transit.\\
        \midrule
        x & 1.365 & 2.249 & 0.183 & 0.098 & 0.295 \\
        y & 1.334 & 2.004 & 0.162 & 0.089 & 0.595 \\
        z & 0.979 & 1.286 & 0.183 & 0.092 & 0.529 \\
        \bottomrule
    \end{tabular}
    \caption{The table shows the RMSE for the comparison of the raw magnetometer and velocity measurements and the calculated measurement estimates.
Measurement estimates are generated based on the identified sensor model, the parameterization, and given core states. The RMSE is shown for the initial and middle segments of the trajectory.}
    \label{tab:smi_realworld_rmse}
    \vspace{-3mm}
\end{table}

The rigid body translation between the IMU sensor and the velocity sensor $\vvar{p}[i][v]$ also shows plausible error margins.
In general, the error for the z-axis is higher ($\sim \SI{20}{\centi\meter}$) for both data segments. This is due to the fact that the true calibration of the sensor is negligible for its z-component.
Because of this, a high angular velocity around a directional vector from the IMU to the sensor would be required to correctly estimate this parameter \cite{Huang2024}. Yet, the high leaver-arm of \SI{0.56}{\meter} of the x and y-component allows for well-observable x and \mbox{y-calibration}, shown by the low error for the initial phase. On the contrary, the y-component for the mid-flight data segment shows a higher error due to the lack of roll and pitch excitation during flight.
Despite the small errors for the calibration states, the overall RMSE, shown by Table~\ref{tab:smi_realworld_rmse} for the individual data segments, is reasonable. 

\begin{figure}[t]
    \centering
    \vspace{0mm}
    \includegraphics[width=1\linewidth, trim={0 0mm 0 0mm},clip]{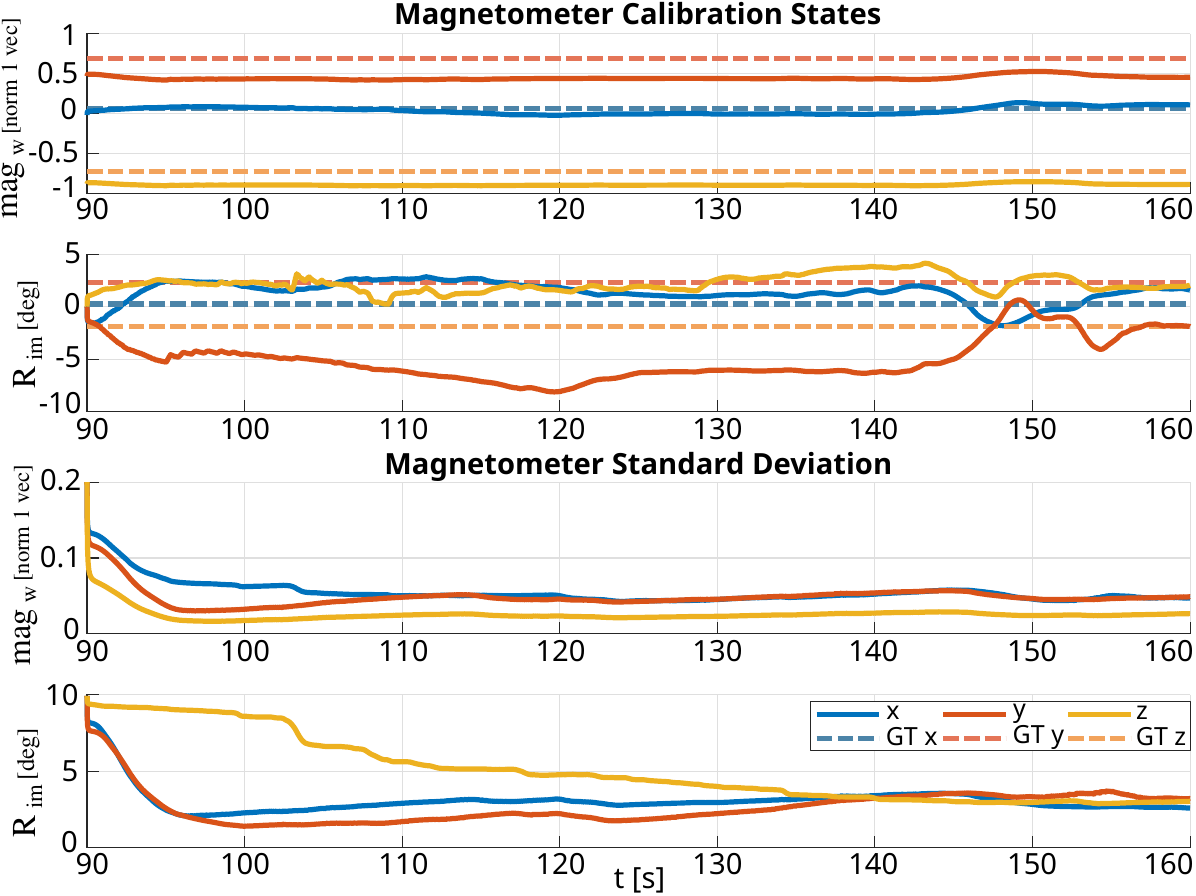}
    \caption{The graphs show the convergence of the magnetometer calibration states: the reference of the magnetic field expressed in the world frame $\vvar{mag}[w]$, and the rigid body magnetometer sensor orientation with respect to the IMU sensor $\rvar{i}{m}$.
    The magnetometer is added to the MaRS framework during runtime at $t=\SI{90}{\second}$ with the initial guess of the sensor model identification approach. This plot shows that the calibration states converge toward the correct values despite an initial estimate with a slightly higher error, as shown by Table~\ref{tab:smi_realworld_mag}. This scenario is particularly difficult because the calibration states require more agile rotational movement to converge correctly, and the given trajectory only has moderate movement.
    }
    \label{fig:smi_realworld_mars_mag}
    \vspace{-5mm}
\end{figure}

As mentioned in the introduction, the sensor model identification approach is meant to automate the process of integrating a sensor into a state estimation framework, possibly to the extent where this is done online. Thus, in the following, we are using the MaRS framework, which utilizes two GNSS sensors for the estimate of the core states during the whole sequence of the trajectory, and introduce the previously identified sensor models, in two separate experiments, with the initial guess from the calibration stage at $t=\SI{90}{\second}$.

\begin{figure}[tb]
    \centering
    \includegraphics[width=1\linewidth, trim={0 0mm 0 0mm},clip]{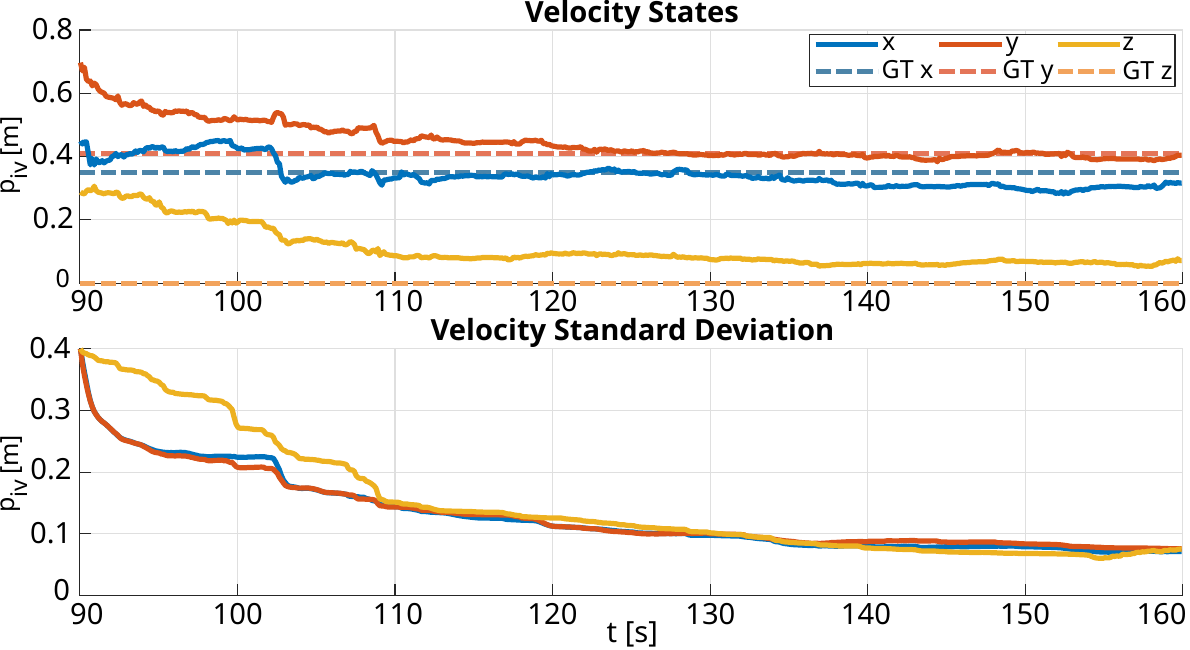}
    \caption{Convergence of the rigid body velocity calibration states with the position of the sensor with respect to the IMU frame $\vvar{p}[i][v]$.
    The sensor is added to the MaRS framework during runtime at $t=\SI{90}{\second}$ with the initial guess of the sensor model identification approach. This plot illustrates that the calibration states converge toward the correct values despite a higher error for the initial guess, as shown by Table~\ref{tab:smi_realworld_vel}.
    }
    \label{fig:smi_realworld_mars_vel}
    \vspace{-5mm}
\end{figure}

The result for the integration of the magnetometer sensor is shown by Figure~\ref{fig:smi_realworld_mars_mag}. This graph illustrates the convergence of the magnetic field vector $\vvar{m}[w]$ and rigid-body rotation $\rvar{i}{m}$. It can be seen that both calibration states converge towards the true value.
This is an improvement compared to the initial guess but not a convergence as might be expected. There are two reasons for this slight discrepancy.
One is the lack of orientational excitation, which limits the quality of convergence. The second is more predominant and due to the intrinsic calibration of the magnetometer sensor.
It needs to be noted that the normalized measurements of a magnetic field vector lie on a unit sphere; however, magnetometer suffer from hard- and soft-iron errors which cause an offset and distortion of this sphere and require dedicated compensation \cite{Vasconcelos2011}. 
This compensation is done for the generation of the ground truth calibration of the sensor, but not for the measurement input to the sensor model identification approach. This is done to maintain a fair comparison because the input measurement is meant to be unknown and unprocessed.

The result of integrating the velocity sensor is shown by Figure~\ref{fig:smi_realworld_mars_vel}. The distinct changes in velocity (see Fig.~\ref{fig:smi_realworld_vel}) result in a clean convergence of the rigid body calibration $\vvar{p}[i][v]$, which compensates for the small error of the initial guess given by the sensor model identification approach.

Overall, the given sensor types have been reliably detected in both real-world cases.
The boolean selectors show a high value, which indicates a low ambiguity in the selection process despite the fact that the data was prone to observability issues for both sensor types.
The estimate for the initial guess of the sensor calibrations showed small errors after completing the sensor parameterization stage. However, the consecutive integration of the parameterized sensor models to a state estimator during runtime showed that the self-calibration of the sensor states further improve the sensor's rigid body calibrations.

This method can be improved by using observability-aware motion generation \cite{Hausman2017}, either during the time of the sensor model identification or the later refinement of the calibration states by the recursive filter. However, this is out of the scope of the presented experiment.

The previous experiments showed an example with undisturbed real-world sensor data. The following experiment will show that the identification approach can also identify a model if real-world data with an increased noise profile is given. For this case, we are using the same dataset and choose the outdoor-to-indoor transition sequence between $120\text{-}\SI{170}{\second}$ to identify a velocity sensor.

During this time, the vehicle is moving along the front of a building and enters the building at the end.
Thus, the raw velocity measurement (see Fig.~\ref{fig:smi_realworld_vel}) shows a higher noise due to shadow-casting and multi-path effects (urban canyoning) and distinct outliers at the end of this sequence because the GNSS signal is dropping out. Therefore, one part of the given measurement sequence is ill-conditioned and the second part is becoming invalid.

Please note that the estimate of the vehicle state is based on a multi-sensor and sensor switching setup including GNSS, VIO, Magnetometer, and an indoor motion capturing system. The interested reader is referred to \cite{Brommer2024_insane} for the detailed setup of the filter.
Thus, the estimate of the vehicle localization can be assumed to be correct for the entire period and only the GNSS data stream becomes invalid at the end.

\begin{figure}[!t]
    \centering
    \includegraphics[width=1.0\linewidth, trim={0 0mm 0 0mm},clip]{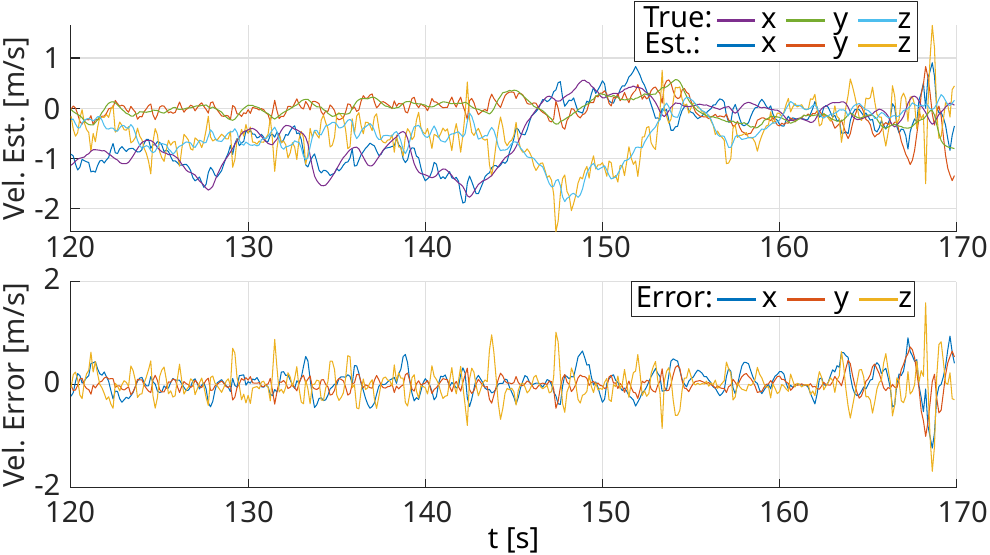}
    \caption{
    Reprojection of the correctly identified velocity sensor model and its ground truth as well as the corresponding error plot. Despite the fact that the given real-world measurement sequence has increased noise and partially invalid measurements at the end, the resembled data generally matches the true signal.
    However, the plots show jitters due to the increased error for the estimated calibration states, which is a result of the erroneous real-world measurement sequence.}
    \label{fig:smi_real_world_vel_transition}
    \vspace{-5mm}
\end{figure}

Table~\ref{tab:smi_boolean_realworld} shows that the velocity sensor model is correctly identified, confirmed by the boolean selector. Compared to the results for the initial and middle segment, Table~\ref{tab:smi_realworld_vel} shows an increased error for the estimation of the calibration states for the transition scenario.
This is due to the higher noise floor and lack of orientation, leading to poorly observable calibration states. Table~\ref{tab:smi_realworld_rmse} also shows an increased RMSE due to the significantly higher noise of the measurement signal that we are comparing against.
However, the correct sensor model identification shows that the proposed approach also works with realistically ill-conditioned real-world data. Figure~\ref{fig:smi_real_world_vel_transition} shows the corresponding overlay of the ground truth for the measurement signal, and the estimate, based on the identified sensor model and parameters for the calibration states as well as the error over time.
\subsection{Real-World Application (Ground Vehicle)}
\label{sec:experiment_real_world_ground_vehicle}

While the previous section demonstrates the approach to sensor model identification for a \ac{UAV} use case, this section presents an experiment for a ground vehicle. We are using the same filtering framework (MaRS \cite{Brommer2021_mars}), which uses an IMU as the primary sensor and therefore renders the state dynamics agnostic to the robotics platform.
Thus, the estimates of both vehicle localizations scenarios are similar in characteristics. However, the importance of this experiment lies in the profile and dynamic of the trajectories for the robotic vehicle itself. In this case, the trajectory is moderate in excitation and bound to the ground plane which poses a significant aspect for the validation of the proposed algorithm.

We are using sensor data from a car that drove monotonously (i.e., cruise control) with an average speed of \SI{55}{\kilo\meter\per\hour} along a trajectory of \SI{64}{\kilo\meter} with mainly straight elements and three dedicated turning points as shown by Figure~\ref{fig:car_3d_trajectory}.
The car was equipped with an IMU (@\SI{200}{\hertz}), GNSS (@\SI{5}{\hertz}) measuring position and velocity, and a magnetometer (@\SI{40}{\hertz}).
The IMU and GNSS information is used for the base of the localization algorithm, using the MaRS framework, which provides the system input to the sensor model identification approach.

\begin{figure}[!t]
    \centering
    \includegraphics[width=1.0\linewidth, trim={0 0mm 0 0mm},clip]{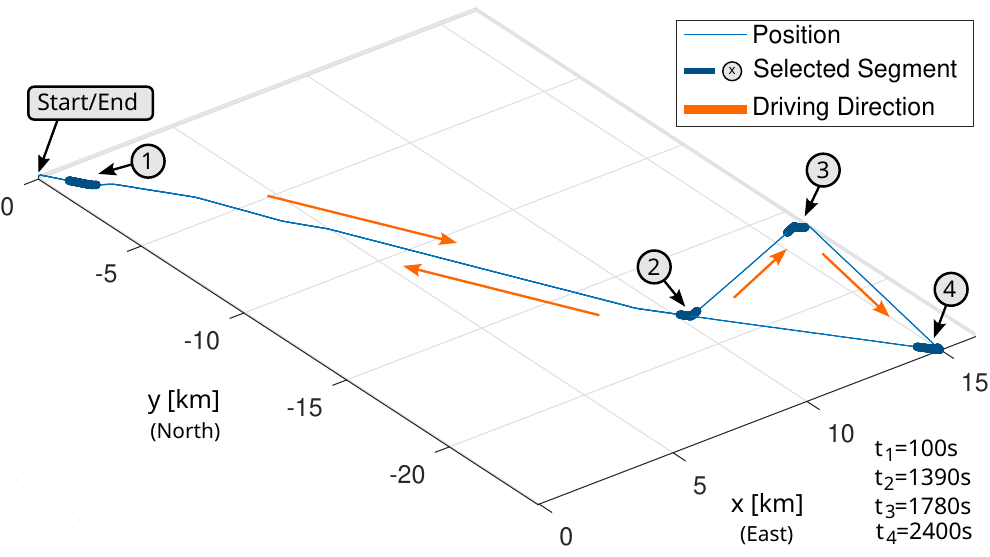}
    \caption{
    3D trajectory of the ground vehicle position estimated with the MaRS framework. Bold segments highlight the four evaluation sequences for the sensor model identification.}
    \label{fig:car_3d_trajectory}
\end{figure}

\begin{figure}[!ht]
    \centering
    \includegraphics[width=1.0\linewidth, trim={0 0mm 0 0mm},clip]{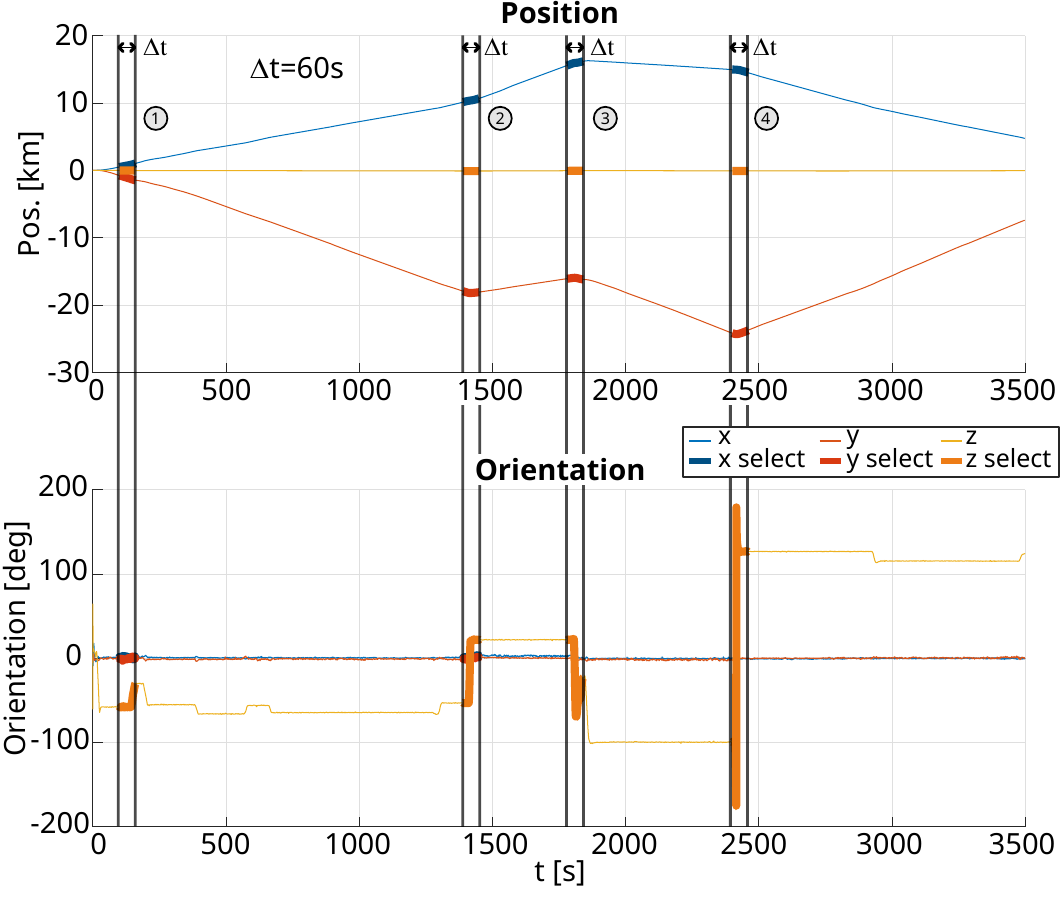}
    \vspace{-6mm}
    \caption{2D position and orientation of the estimated ground vehicle trajectory with highlighted sections corresponding to sequences in Figure~\ref{fig:car_3d_trajectory}.}
    \label{fig:car_2d_trajectory}
\end{figure}
\begin{figure}[!ht]
    \centering
    \vspace{-4mm}
    \includegraphics[width=1.0\linewidth, trim={0 0mm 0 0mm},clip]{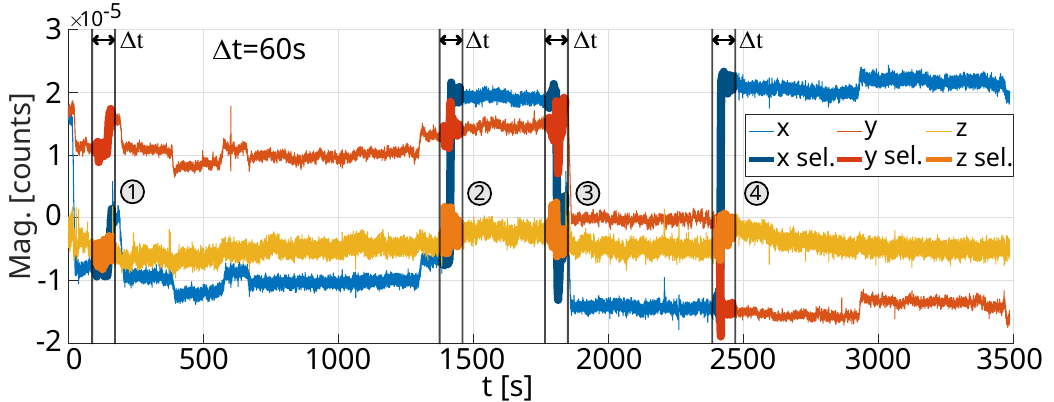}
    \caption{Raw magnetometer sensor measurement with highlighted sections used for the sensor model identification approach.}
    \label{fig:car_sensor_data}
    \vspace{-6mm}
\end{figure}

For the overall experiment, we are using the magnetometer measurements of the highlighted $\Delta t=\SI{60}{\second}$ trajectory segments for the sensor model identification.
The estimated position and rotation of the trajectory is highlighted in Figure~\ref{fig:car_2d_trajectory}, with corresponding magnetometer measurements shown by Figure~\ref{fig:car_sensor_data}.
All segments have rotational information in yaw and negligible excitation in the roll and pitch direction since the vehicle is bound to the ground and the experiment area was flat.
Segment one contains a short turn of $\SI{45}{\degree}$, segments two and three have a turn of close to $\SI{90}{\degree}$, and segment four has a turn of $\SI{120}{\degree}$.

\begin{table}[!h]
    \centering
    \begin{tabular}{ r c c c c }
        \textbf{Soft Boolean}& \textbf{Seq. 1} & \textbf{Seq. 2}  & \textbf{Seq. 3} & \textbf{Seq. }4 \\
        \midrule
        Selector        & 1.149 & 0.684 & 0.947 & 0.902\\
        Dist. to second & 1.029 & 0.527 &  0.974 &  0.681\\
        \bottomrule
    \end{tabular}
    \caption{The table shows that the magnetometer sensor model is correctly detected for all four sequences. The magnitude of the soft Boolean selectors show a clear selection, and the distance to the second highest boolean is above the threshold of $\Delta \mathsf{b} > 0.31$, which is a measure for the accuracy of the selection and confirms that the selections are not false positives (see Sec.~\ref{sec:healt_metric}).}
    \label{tab:smi_boolean_realworld_car}
    \vspace{-0.4cm}
\end{table}

\begin{table}[!h]
    \centering
    \setlength\tabcolsep{3pt}
    \begin{tabular}{cccccc}
       & & \multicolumn{4}{c}{\textbf{Error of Estimate}}\\
       \cmidrule(lr){3-6}
       $\rvar{i}{m}$ \textbf{[deg]} & \textbf{True} & {Seq. 1} & {Seq. 2} & {Seq. 3} & {Seq. 4}\\
       \midrule
       x & ~4 & 4 &  4 & 4 & ~4  \\
       y & ~17 & 17 & 17 & 17 & ~17\\
       z & -1 & -1 & -1 &  -1 & -1\\
       \midrule
       $\vvar{m}[w]$ \textbf{[deg]} & & & & & \\
       Incl. & -16.5 & -2.5 & 12.0 & 9.8 & 10.0\\
       Decl. & 75.5 & -6.2 & -14.6 & -11.6 & 21.2\\
       \bottomrule
    \end{tabular}
    \caption{Comparison of the true magnetic field vector expressed in the world frame $\vvar{m}[w]$ and the rigid-body rotation $\rvar{i}{m}$, to the identified sensor model and estimated calibration states. Estimates of $\rvar{i}{m}$ are close to zero, which is due to the minimal orientational excitation of the vehicle and low quality of observability.
    }
    \label{tab:smi_realworld_error_car}
\end{table}
\begin{table}[!h]
    \centering
    \vspace{-5mm}
    \begin{tabular}{ c c c c c }
         & \multicolumn{4}{c}{\textbf{Mag.} [$\mu$ counts]}\\
        \cmidrule(lr){2-5}
        \textbf{RMSE} & {Seq. 1} & {Seq. 2}  & {Seq. 3} & {Seq. 4} \\
        \midrule
        x & 1.196 & 3.531 & 3.176 & 4.141 \\
        y & 1.187 & 3.513 & 3.005 & 4.051 \\
        z & 0.681 & 1.213 & 1.615 & 1.154 \\
        \bottomrule
    \end{tabular}
    \caption{The table shows the RMSE for the comparison of the raw magnetometer measurements vs. the calculated measurement estimates, based on Equation~\ref{eq:magnetometer} and the parameterization determined by the identification process.}
    \label{tab:smi_realworld_rmse_car}
    \vspace{-5mm}
\end{table}

These segments are intentionally chosen because they contain rotational information. This is justified by the earlier statement that the properties of the trajectory need to render the states of the sensor observable, which is given by the orientational excitement in the turning points of the trajectory. Otherwise, the self-calibration state $\mathbf{m}_\mathsf{w}$ shown in Equation~\eqref{eq:magnetometer} is unobservable. We refer the interested reader to previous work in \cite{Brommer2020_mag} for a detailed analysis.

The results show that the sensor model identification approach correctly identified the magnetometer model based on the given measurement data for all four experiment sequences. The correct identification is validated by table~\ref{tab:smi_boolean_realworld_car} showing that the soft Boolean selector provides a clear decision with an unambiguous distance to the second soft boolean.

Please keep in mind that this particular example is meant to show that a sensor model, such as the magnetometer, can also be identified with insufficient rotational excitement in all three axes, posing a low quality in observability of the calibration states. While the identification of the model was correct for the presented experiments, the lack of roll and pitch excitement did not allow to clearly restrain the relation between $\rvar{i}{m}$ and $\vvar{mag}[w]$ causing higher errors for the estimated calibration states. This is shown by table~\ref{tab:smi_realworld_error_car} and \ref{tab:smi_realworld_rmse_car} respectively.

The estimates of $\vvar{mag}[w]$ and $\rvar{i}{m}$ show an error in the range of \SI{20}{\degree}, which may be acceptable for the initialization of a filter if the remaining trajectory would provide sufficient orientations to allow the convergence of the calibration states.
However, since the properties of the trajectory are known by design (i.e. a car on flat grounds), and the requirements on observability for the magnetometer model are known, a high-level decision can be made to only include the sensor upon rotational excitement.

\FloatBarrier

\section{Future Work}
\label{sec:future_work}
In terms of future work, we showed that it is possible to identify sensors based on an extended catalog of known sensor models. This could be extended by future work to explore a more granular approach. Lessons learned while developing the current approach are, that more granular selectors for sensor model components quickly lead to model over-determination and non-applicable sensor models due to observability constraints caused by shared observability of co-dependent states.
We evaluated a system setup where each sensor state was defined as a shared state between models, but the results were limited in terms of robustness.
In addition, if a sensor model has additive components that can be obscured by the noise floor (SNR), then detection against other model hypotheses is difficult and not robust because they are not distinguishable.
While this indicates that the system definition can not be oversimplified, future research could further isolate components of the measurement information to be more decoupled from known measurement models, and design models that include latent, non-intuitive states.
The presented approach can also be extended and adopted to address challenges in fields outside of the robotics community e.g. neuroscience, bio-engineering, or logistics.
%

\vspace{-3.5mm}
\section{Conclusion}
\label{sec:conclusion}
The introduced approach allows us to infer a sensor model by providing the current state of a vehicle and a corresponding series of sensor measurements. The selection can be validated by the introduced health metric, defined and validated by an extensive cross-examination of final optimization losses and selection parameters, utilizing a scatter-plot matrix to maximize the recall of the metric at a \SI{100}{\percent} precision.
The approach also allows us to reliably decide if a reference frame for sensor measurements is required or not. As a result, the removal or introduction of a reference frame as a possible calibration state showed improved accuracy for the remaining calibration parameters in both cases.
The inference of an initial guess for state calibrations of the formally unknown measurement signal is sufficiently accurate for the initialization of a localization algorithm, which was proven through the application in a real-world scenario.
We did not identify any severe shortcomings. The exposed failure modes only occur in scenarios in which the measurement noise of individual components is unrealistically high.

This method is helpful for inexperienced users who need to identify the source and type of measurement, sensor calibrations, or sensor reference frames. 
It will also be important for the automated integration of sensor modalities in the field of modular multi-agent scenarios and modularized robotic platforms that are augmented by sensor modalities during runtime, ensuring a more accurate and robust integration of new sensor elements.
Overall, this approach aims to provide a simplified integration of sensor modalities to downstream applications and circumvent common pitfalls in the usage and development of localization approaches.
%

\bibliographystyle{IEEEtran}
\bibliography{bib/christian}

\vspace{8.5mm}
\section*{Biographies}
\vspace{-8mm}
\begin{IEEEbiography}
[{\includegraphics[width=1in,height=1.25in,clip,keepaspectratio]{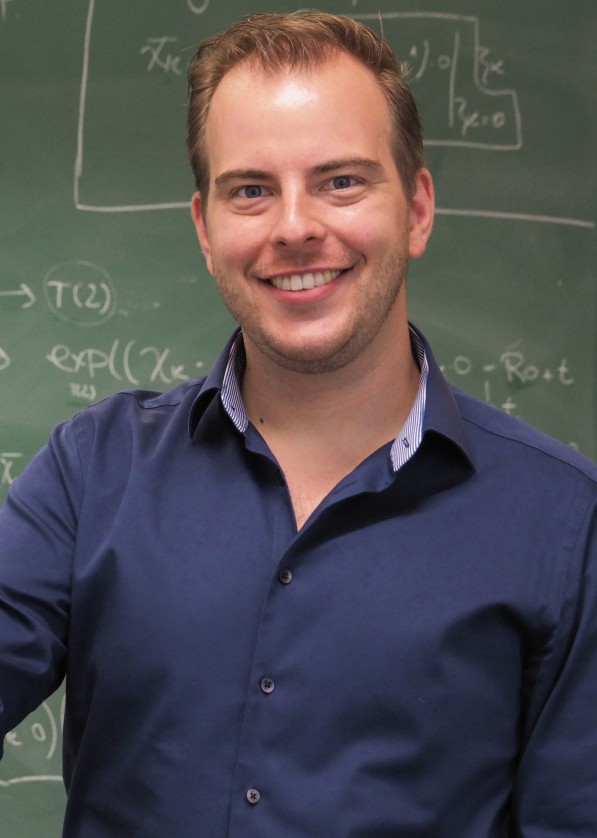}}]
{Christian Brommer} received the Ph.D. degree in robotics from the University of Klagenfurt, at the Control of Networked Systems group, Austria, in 2024. He is currently a Postdoctoral Fellow with Oak Ridge Associated Universities at the DEVCOM Army Research Laboratory. From 2014 to 2018, he conducted research on UAV localization and autonomy at NASA’s Jet Propulsion Laboratory. His research interests include modular multi-sensor fusion and autonomy for robotic vehicles and swarms.
\end{IEEEbiography}

\vspace{-6mm}

\begin{IEEEbiography}[{\includegraphics[width=1in,height=1.25in,clip,keepaspectratio]{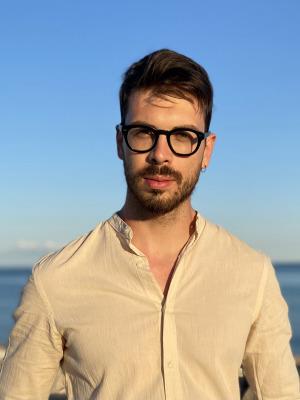}}]{Alessandro Fornasier} received the Ph.D. degree in robotics from the University of Klagenfurt, at the Control of Networked Systems group, Austria, in 2024, where he focused on equivariant aided inertial state estimation and multi-sensor fusion. He is currently a senior Robotics Engineer with Hexagon Robotics. His research interests include aided inertial localization and mapping, sensor fusion, and spatial AI for autonomous systems.
\end{IEEEbiography}

\vspace{-6mm}

\begin{IEEEbiography}[{\includegraphics[width=1in,height=1.25in,clip,keepaspectratio]{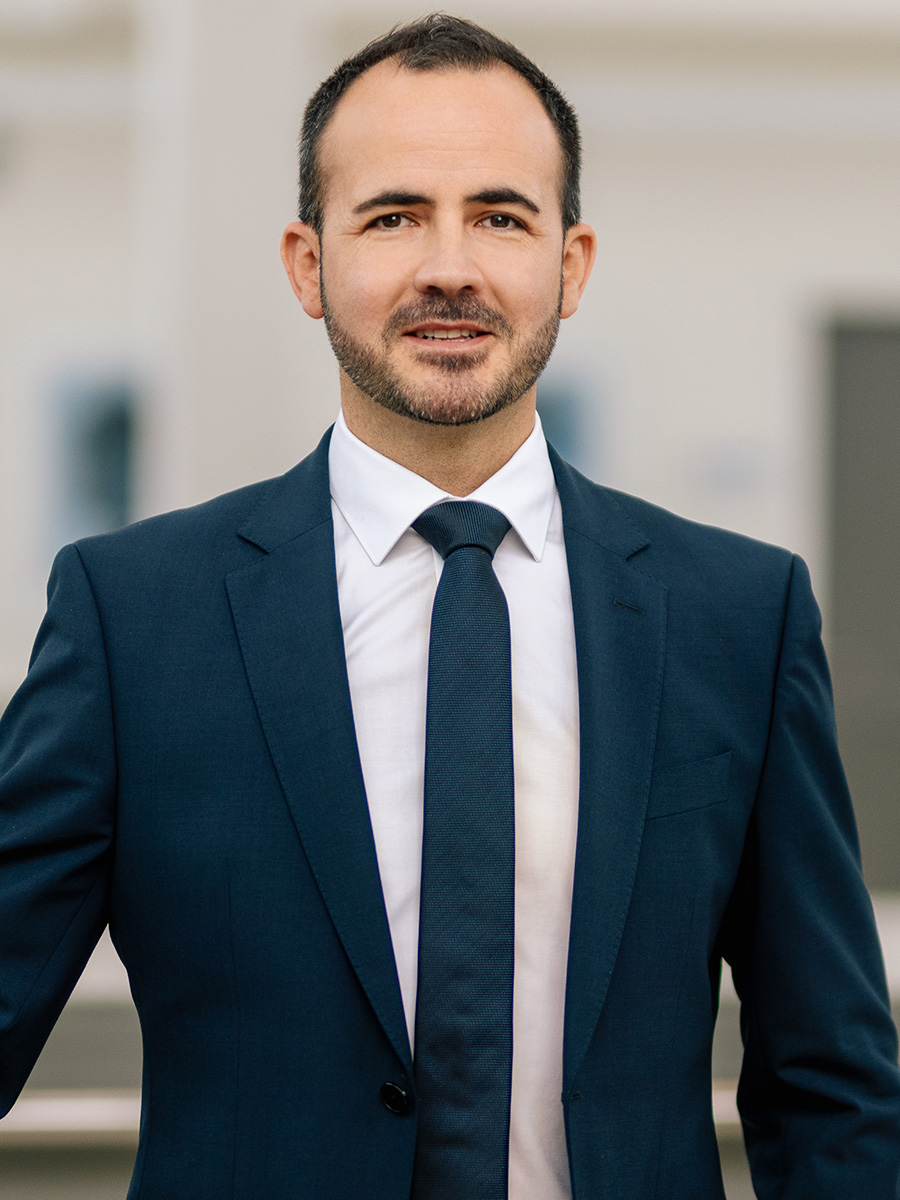}}]{Jan Steinbrener} received the Ph.D. degree in physics from Stony Brook University, NY, USA, in 2010, where he worked on inverse methods in X-ray diffraction microscopy. He since had various positions in industry and academia and is currently Vice Rector of Research and International Affairs and an Associate Professor with the department of Smart Systems Technologies at the University of Klagenfurt, Austria. His research focuses on AI-based methods for control and navigation of robotic systems.
\end{IEEEbiography}

\vspace{-6mm}

\begin{IEEEbiography}[{\includegraphics[width=1in,height=1.25in,clip,keepaspectratio]{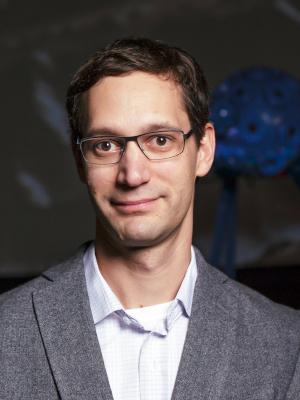}}]{Stephan Weiss} received the Ph.D. degree in robotics from ETH Zurich, Switzerland. He is currently a Full Professor at the University of Klagenfurt, Austria, where he leads the Control of Networked Systems Group. He was a Research Technologist with the mobility and robotic systems section at the NASA Jet Propulsion Laboratory where he took part in initiating the Ingenuity Mars helicopter project. His research interests include vision-based navigation, visual-inertial state estimation, and autonomous aerial vehicles.
\end{IEEEbiography}

\vfill

\end{document}